\documentclass[twoside]{article}
\PassOptionsToPackage{dvipsnames,table}{xcolor}
\usepackage[accepted]{aistats2024}
\usepackage{microtype}
\usepackage{graphicx}
\usepackage{subcaption}
\usepackage{booktabs, multirow}  
\usepackage{soul}

\usepackage{hyperref}

\usepackage{bm}
\usepackage{amsfonts, amsmath, amsthm}

\def\eqref#1{equation~\ref{#1}}

\def\1{\bm{1}}

\def\rvx{{\mathbf{x}}}

\def\rvz{{\mathbf{z}}}

\def\vx{{\bm{x}}}
\def\vy{{\bm{y}}}
\def\vz{{\bm{z}}}

\DeclareMathAlphabet{\mathsfit}{\encodingdefault}{\sfdefault}{m}{sl}
\SetMathAlphabet{\mathsfit}{bold}{\encodingdefault}{\sfdefault}{bx}{n}

\def\gC{{\mathcal{C}}}
\def\gD{{\mathcal{D}}}

\def\gF{{\mathcal{F}}}
\def\gG{{\mathcal{G}}}
\def\gH{{\mathcal{H}}}

\def\gM{{\mathcal{M}}}

\def\gX{{\mathcal{X}}}
\def\gY{{\mathcal{Y}}}
\def\gZ{{\mathcal{Z}}}

\def\sD{{\mathbb{D}}}
\def\sF{{\mathbb{F}}}

\def\sR{{\mathbb{R}}}

\newcommand{\E}{\mathbb{E}}

\newcommand{\R}{\mathbb{R}}

\newcommand{\softmax}{\mathrm{softmax}}
\newcommand{\sigmoid}{\sigma}

\DeclareMathOperator*{\argmax}{arg\,max}
\DeclareMathOperator*{\argmin}{arg\,min}
\newtheorem*{rep@theorem}{\rep@title}
\newcommand{\newreptheorem}[2]{%
\newenvironment{rep#1}[1]{%
 \def\rep@title{#2 \ref{##1}}%
 \begin{rep@theorem}}%
 {\end{rep@theorem}}}
\makeatother
\newtheorem{theorem}{\protect\theoremname}
\providecommand{\theoremname}{Theorem}
\newreptheorem{theorem}{Theorem}
\newtheorem{proposition}{Proposition}
\newtheorem{lemma}{Lemma}
\theoremstyle{definition}
\newtheorem{definition}{Definition}[section]

\usepackage{booktabs}       
\usepackage{nicefrac}       
\usepackage{placeins}
\usepackage{microtype}      
\usepackage{amssymb}
\usepackage{xcolor}
\usepackage{framed}
\usepackage{booktabs}
\usepackage{bbm} 
\definecolor{shadecolor}{rgb}{1,0,0}
\usepackage{microtype}
\usepackage{wrapfig}
\usepackage{changepage,threeparttable}
\usepackage{adjustbox}
\usepackage{algorithm, algpseudocode}

\newcounter{phase}[algorithm]
\newlength{\phaserulewidth}
\newcommand{\setphaserulewidth}{\setlength{\phaserulewidth}}
\newcommand{\phase}[1]{%
  \vspace{-1.25ex}
  \Statex\leavevmode\llap{\rule{\dimexpr\labelwidth+\labelsep}{\phaserulewidth}}\rule{\linewidth}{\phaserulewidth}
  \Statex\strut\refstepcounter{phase}\textit{Phase~\thephase~--~#1}
  \vspace{-1.25ex}\Statex\leavevmode\llap{\rule{\dimexpr\labelwidth+\labelsep}{\phaserulewidth}}\rule{\linewidth}{\phaserulewidth}}
\makeatother

\setphaserulewidth{.7pt}

\begin{document}
\twocolumn[
\aistatstitle{Continual Domain Adversarial Adaptation via Double-Head Discriminators}
\aistatsauthor{ Yan Shen \And Zhanghexuan Ji \And  Chunwei Ma \And Mingchen Gao}

\aistatsaddress{ University at Buffalo \And  University at Buffalo \And University at Buffalo \And University at Buffalo } 
]
\begin{abstract}
    Domain adversarial adaptation in a continual setting poses a significant challenge due to the limitations on accessing previous source domain data. Despite extensive research in continual learning, the task of adversarial adaptation cannot be effectively accomplished using only a small number of stored source domain data, which is a standard setting in memory replay approaches. This limitation arises from the erroneous empirical estimation of $\gH$-divergence with few source domain samples. 
    To tackle this problem, we propose a double-head discriminator algorithm, by introducing an addition source-only domain discriminator that are trained solely on source learning phase. We prove that with the introduction of a pre-trained source-only domain discriminator, the empirical estimation error of $\gH$-divergence related adversarial loss is reduced from the source domain side. Further experiments on existing domain adaptation benchmark show that our proposed algorithm achieves more than 2$\%$ improvement on all categories of target domain adaptation task while significantly mitigating the forgetting on source domain. 
\end{abstract}
\section{Introduction} \label{sec:intro}

Unsupervised Domain adaptation(UDA) refers to the process of transferring knowledge from a labeled source domain to an unlabeled target domain~\cite{ben2010theory,zhao2019learning}, taking into account the presence of {\it domain shifts} between the source and target domains. One line of UDA work to bridge the domain gap focuses on learning domain invariant feature representations by adversarial adaptations\cite{ganin2016domain,zhao2018adversarial, long2017conditional,zhang2019bridging}. Classic adversarial domain adaptation applies in offline settings where both the source and target domain data can be accessed i.i.d. However in the context of continual learning(CL), domain data is considered to be accessed in a sequential manner. The sequential nature of CL makes the direct application of this line of approaches particularly challenging.

Intuitively, one can expect that the gap between offline and online learning would be partly bridged if a small portion of the previous domain data is stored and subsequently accessible. This 'divide-and-conquer' idea has brought up to a setting known as {\it memory replay continual learning} where the learner stores a small portion of previous tasks in memory and replays them with the new mini-batch data. However different from memory replay CL in supervised task\cite{belouadah2019il2m, zhao2020maintaining, hou2019learning, castro2018end, wu2019large, liu2021adaptive}, adversarial adaptation requires estimation of an extra domain discrepancy term, as the $\gH$-divergence, in addition to the supervised task risk on the previous source domain. Prior theoretical results showed that empirically estimating $\gH$-divergence using only a few source samples results in a significant error gap from the source side\cite{ben2010theory}. As a consequence, the model adversarially trained on few stored source samples, would exhibit poorer performance in target adaptation.

In light of the above unique challenge in adversarial adaptation under CL settings, to construct a low-error empirical estimation of domain discrepancy with few source samples, we propose our \textbf{double-head discriminator algorithm}. We train two domain discriminators on domain data of different phases. One is trained in source learning phase as {\it source-only domain discriminator}. The other one is adversarially trained in target adaptation phase with task model.  
And we use the ensemble of two domain discriminators for a lowered empirical error estimation of $\gH$-divergence. In particular, the source-only domain discriminator is trained exclusively with source domain data in one-class learning approaches. It serves as a score-based function to assess the level of in-distribution within the source domain. In the target adaption phase, the source-only domain discriminator is freeze. The ensembles of two domain discriminator's digits are used as $\gH$-divergence signal to learn a domain generalized task model. 


To summarize, our contributions are listed as follows 

\textbf{Contributions}

(i) We propose, a double-head discriminator algorithm, tailored for adversarial adaptation in a CL setting. Different from existing works on continual UDA, our algorithm learns a domain generalized task model with better performance on target domain task while mitigating the issue of catastrophic forgetting on the tasks of previous source domains. Our proposed algorithm is effective, requiring only a few source domain samples stored in the replay memory buffer.

(ii) We theoretically analyze our proposed algorithm. Firstly we show that the population form of two discriminator's ensemble digit does construct a $\gH$-divergence to bound on the generalization error between source and target domain's population risk. Next, we demonstrate that in empirical form, the ensemble of two discriminators reduces the error of empirical estimation on $\gH$-divergence from the source domain side. Finally, we analyze the equilibrium of our adversarial loss on how source only domain discriminator regulates on source and target domain's distributions.


(iii) Empirically, we show that our algorithm consistently achieves better performance on continual adaptation to target domain task while significantly mitigating the issue of catastrophic forgetting on previous source domain tasks. 
\setlength\intextsep{-10pt} 

\section{Preliminary}


\paragraph{Continual Unsupervised Domain Adaptation}

In Continual UDA, the data comes as a stream $S_0, T_1$ and an unified model is trained on current data locally without revisiting previous data. Let $P=\{S_0, T_1\} $ be the data stream, in which $S_0 = \{(\vx^s_i, y^s_i)\}$ contains labeled examples in source domain, $T_1 = \{(\vx^t_i)\}$ contains unlabeled examples in target domain, where $\vx^s_i, \vx^t_i \in \sD,y_{i} \in \gC$. Specifically, the continual domain adaptations algorithm begins with training a task model $f_w$ on labeled examples from source domain $S$. For the successive phase, the data comes as an unlabeled examples on target domains $T$.
We name the two phase as source training phase $S_0$ and target adaptation phase $T_1$. The unified task model $f_{\omega} = f_{\omega}^2\circ f_{\omega}^1$ consists of a {\it feature extractor} $f^1_{\omega}$ and a {\it label predictor} $f_{\omega}^2$. The feature extractor is a deep neural network $\vz = f^1_{\omega}(\vx), \vz \in \sF$ that maps the data to feature space. Continual domain adaptations aims to learn a feature extractor $f^1_{\omega}$ that generates domain invariant feature representations. The label predictor is another network $\vy = f^2_{\omega}(\vz), \vy \in \sR^{|\gC|}$ maps from feature space to task digits space. Both feature extractor and label predictor are trained continuously at both $S_0$ and $T_1$ phase. In addition to the task model, the trained model also involves another {\it domain discriminator} network  $d = h_{\psi}(\vz), d \in \sR$ that tries to determine whether the
extracted features belongs to the source or target domain. 
\paragraph{Target Domain Adaptation with Few Stored Source Samples is Challenging}
Memory replay in continual learning methods refers to the practice of storing a small number of samples from previous domains and replaying them alongside the current data stream in mini-batches during the learning process on new domain data. We denote memory buffer as $\gM$ that stored few portions of previously accessed domain data. With the incorporation of replay memory buffer $\gM$, it optimizes on a empirical task loss of the joint distribution of the current data stream and the replay memory $\gM$\cite{chaudhry2019tiny}. In traditional supervised CL settings, the empirical task loss on $\gM$ is trained purely for memorizing the old task. Uniquely in adversarial adaptation, the new task objective of target adaptation on $T_1$ takes the general form of an empirical domain adversarial loss on the joint distribution of target domain data in $T_1$ and stored source domain data in $\gM$, as follows:
\begin{equation}
\begin{split}    &\min_{\omega}\max_{\psi}\mathbb{E}_{(\mathbf{x}^s_i, y^s_i) \sim \gM}[\ell(f_{\omega}(\mathbf{x}^s_i), y^s_i) 
\\&- \nu \mathbb{E}_{(\mathbf{x}^s_i, y^s_i) \sim \gM}D^s_{\psi}(\mathbf{x}^s_i)] - \nu\mathbb{E}_{\mathbf{x}^t_i \sim T_1}D^t_{\psi}(\rvx^t_i)
\end{split}
\end{equation}
As a part of new adversarial adaptation task on target domain at $T_1$, $\mathbb{E}_{(\mathbf{x}^s_i, y^s_i) \sim \gM}D^s_{\psi}(\mathbf{x}^s_i) + \mathbb{E}_{\mathbf{x}^t_i \sim T_1}D^t_{\psi}(\rvx^t_i)$ is related to an empirical estimation of the $\gH$-divergence $d_{\gH\Delta\gH}$ in Theorem  \ref{ben2010theory} of our appendix. For more detailed introduction on adversarial domain adaptation, we refer interested readers to Appendix (\ref{sec:adv}). However, according to Theorem \ref{empiricalhdiv} given by Ben-David et.al,\cite{ben2010theory}, 
using few samples of stored source domains data to construct an empirical version of $\gH$-divergence, denoted as $\hat{d}_{\gH\Delta\gH}$, can result in significant errors when estimating the population $\gH$-divergence.
\begin{theorem}
\label{empiricalhdiv}
Let $\gF$ be a hypothesis space with VC dimensions $d$, if $S'$ are samples of size $m$ from $S$ and $T'$ are samples of size $n$ from $T$ respectively and $\hat{d}_{\gH\Delta\gH}(S', T')$ is the empirical $\gH$-divergence between samples, then for any $\delta \in (0,1)$, with probability at lease $1-\delta$
    \begin{equation}
    \begin{split}
    &d_{\gH\Delta\gH}(S, T) \leq \hat{d}_{\gH\Delta\gH}(S', T') + 2\sqrt{\frac{d\log2m+\log(2/\delta)}{m}}\\ &+2\sqrt{\frac{d\log2n+\log(2/\delta)}{2n}}
    \end{split}
    \end{equation}
\end{theorem}
Due to the erroneous estimation from the learning objective of the task, the adversarial adaptation task on target domain at $T_1$ is expected to exhibit poor performance.

\section{Methodology} \label{sec:method}

\subsection{Double head Domain Discriminator For Continual UDA}
To compensate for the erroneous empirical estimation of $\gH$-divergence originating from few source domain samples, our natural idea is to introduce an additional domain discriminator that is trained on the full set of source domain data instead of a tiny set on memory buffer. In the specific problem setting of continual UDA, the auxiliary domain discriminator is trained on $S_0$ phase and then frozen during $T_1$ phase. Since only the source domain data is accessible in the $S_0$ phase, the auxiliary domain discriminator we introduced in $S_0$ phase is {\it source-only domain discriminator}. Extending the general loss function of $\gH$-divergence to single-side(source) domain loss leads the following form
\begin{equation}
\begin{split}
\label{sourcediscobj}
&\hat{d}_{\gH\Delta\gH} \triangleq sup_{\psi}[\mathbb{E}_{\rvx^s_i \in S_0}D(\sigmoid(h_{\psi,s}(f^1_{\omega}(\mathbf{x}^s_i)))) \\
&- \mathbb{E}_{\rvx^s_i \notin S_0}D(\sigmoid(h_{\psi,s}(f^1_{\omega}(\mathbf{x}^t_i))))]
\end{split}
\end{equation}
The above training objective $\hat{d}_{\gH\Delta\gH}$ has a similar problem formulation of one-class learning. Specifically, the training data $\rvx_i \in S$ is treated as a one-class distribution. And a score based function $\sigmoid({h_{\psi, s}}\circ f^1_{\omega})(\rvx) \in [0, 1]$ is trained to determine how possible that a data instance $\rvx$ lies within the distribution of training dataset $S_0$(source domain). Though one-class learning doesn't not learn a boundary as distinguishable as multi-class classification model. An ideal one-class score function should exhibit positive correlations on its score with data points that belong to the in-distribution and have higher densities.

Though deep one-class learning problem is a class of challenging task that is still under active research. we will describe our solution in the specific case of source only domain classifier in Section (\ref{secmdd}).

In the remaining part of this section, we will describe how we utilize the two complementary domain discriminators jointly to learn a domain generalized task model in target adaptation phase $T_1$.  After the task model $f_{\omega}$ and source only domain discriminator $h_{\psi, s}$ is trained on source domain task. 
Then in the successive $T_1$ phase, the pre-trained source-only domain discriminator $h_{\psi, s}$ is freeze. We introduced another target adaptation discriminator $h_{\psi, t}$ that is adversarial trained with feature generator $f^1_{\omega}$ during $T_1$ phase. The target adaptation discriminator is trained discriminatively using the features from the source domain memory buffer $\gM$ and the target domain data in $T_1$ with the commonly used cross-entropy loss:
\begin{equation}
\label{psitobj}
\begin{split}
&D_{\psi, t}^s(\mathbf{x}^s_i) = -\log(\sigmoid(h_{\psi, t}(f^1_{\omega}(\mathbf{x}^s_i))), \hspace{2em}\\
&D_{\psi, t}^t(\mathbf{x}^t_i) =- \log(1- \sigmoid(h_{\psi, t}(f^1_{\omega}(\mathbf{x}^t_i)))) \\ 
&\min_{\psi_t}[\mathbb{E}_{\mathbf{x}^s_i \sim \gM}D_{\psi, t}^s(\mathbf{x}^s_i)+ \mathbb{E}_{\mathbf{x}^t_i \sim T_1}
D_{\psi, t}^t(\mathbf{x}^t_i) ]
\end{split}
\end{equation}
 To learn domain-independent feature representations, the feature extractor $f^1_{\omega}$ is trained adversarially with the target domain discriminator $h_{\psi, t}$. The estimated $\gH$-divergence from the domain discriminator is used as a signal to guide the learning of domain-invariant feature representations. Instead of solely relying on the target domain discriminator $h_{\psi, t}$ that is trained with only few samples of source domain data in $\gM$, we utilize the ensembles of source and target domain discriminator outputs to obtain a lower empirical estimation of the $\gH$-divergence between the distributions of the source and target domains. The adversarial loss function for learning feature extractor $f^1_{\omega}$ with respect to $\gH$-divergence is given by:  
 \begin{equation}
 \label{fomegaobj}
\begin{split}
&D_{\psi}^s(\mathbf{x}^s_i) = -\log(\sigmoid(h_{\psi, s}(f^1_{\omega}(\mathbf{x}^s_i))+h_{\psi, t}(f^1_{\omega}(\mathbf{x}^s_i)))\\ &D_{\psi}^t(\mathbf{x}^t_i) =- \log(1- \sigmoid(h_{\psi, s}(f^1_{\omega}(\mathbf{x}^t_i))+h_{\psi, t}(f^1_{\omega}(\mathbf{x}^t_i)))) 
\end{split}
\end{equation}
With the previously mentioned loss function for $\gH$-divergence, the joint learning objective for the task model $f_{\omega}(\cdot)$ during the target adaptation phase $T_1$ can be expressed as follows:
\begin{equation}
\min_{\omega} \mathbb{E}_{(\mathbf{x}^s_i, y_i^s) \sim \gM}[\ell(f_{\omega}(\mathbf{x}^s_i) , y_i^s) -\nu D^s_{\psi}((\mathbf{x}^s_i))]  - \nu \mathbb{E}_{\mathbf{x}^t_i \sim T_1} D^t_{\psi}((\mathbf{x}^t_i))
\end{equation}
The entire diagram for Continual UDA with our double-head discriminator algorithm is illustrated in Fig(\ref{fig:doublehead}).
\begin{figure*}[htb!]
  \begin{center}
    \includegraphics[width=0.95\textwidth]{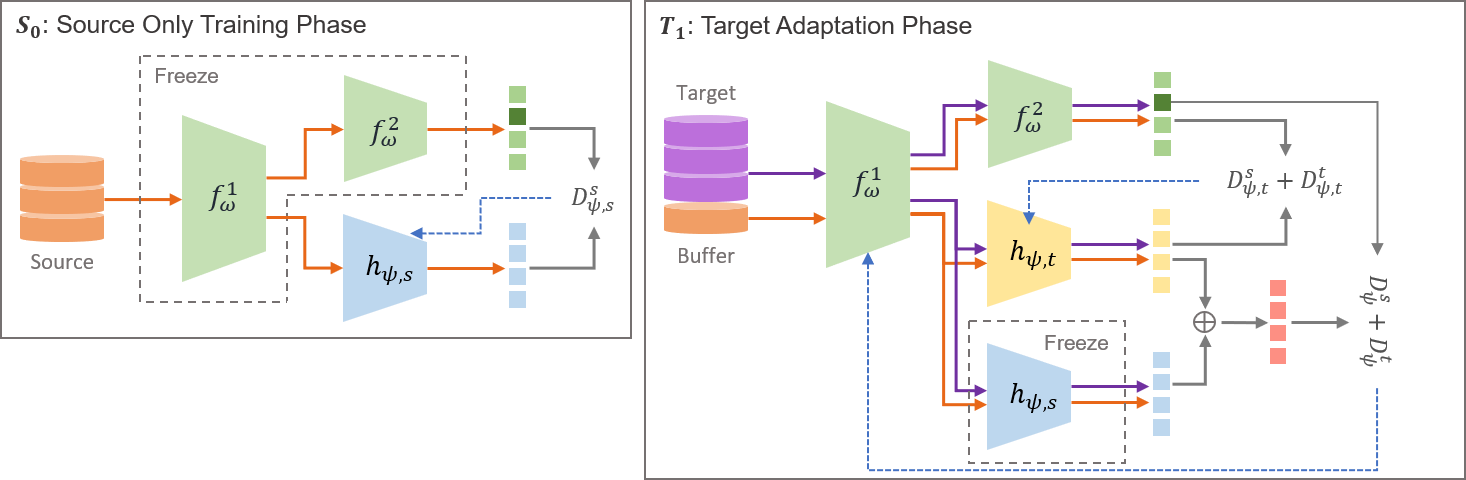}
  \end{center}
  \caption{The flowchart of our proposed double-head discriminator algorithm. The solid line is the forward path. And the dashed line is the backward training path. After the task model is trained in source domain, an additional source-only domain discriminator $h_{\psi,s}$ is trained by freezing the task model $f_{\omega}$. In the target adaptation phase, $h_{\psi,t}$ is adversarially trained with $f_{\omega}^1$ on domain adversarial loss, where the ensembles of domain discriminator $h_{\psi,s}$ and $h_{\psi,t}$'s digit is used as domain adversarial signal to learn domain invariant features for $f_{\omega}^1$}
  \label{fig:doublehead}
 \end{figure*}
\subsection{Example for Single Domain Discriminator Learning: Margin Disparity Discrepancy}
\label{secmdd}
A strait-forward way for one-class learning of source only domain discriminator $h_{\psi, s}$ in (\ref{sourcediscobj}) is optimizing on commonly used cross-entropy function on a single class
\begin{equation}
 \min_{\psi_s}\E_{\rvx_i^s \sim S}[-\log(\sigmoid(h_{\psi,s}(f_{\omega}^1(\rvx_i^s)))]
\end{equation}
However directly training on above objective function would limit the trained source-only domain discriminator's ability as a score based function on in-distribution of source domain. One reason is the uncontrollable digit outputs. And the other reason is the biased features towards highest neuron activations.

One way to address the above limitations of one-class learning is by adding a H-Regularization loss as in \cite{hu2020hrn}. This approach is called HRN and is applicable to general settings of positive, unlabeled learning and continual learning. We would introduce and discuss HRN method for continual UDA in our appendix. However by utilizing the specific problem structure in UDA, a more effective score-based function for domain discrepancy is utilizing the margins between classification spaces as proposed in Margin Disparity Discrepancy(MDD)\cite{zhang2019bridging}. 
Instead of using a binary domain discriminator of scalar outputs $h_{\psi}(\cdot) : \sF \to \sR$, MDD\cite{zhang2019bridging} introduce a multi-class domain discriminator of vector outputs $h_{\psi}(\cdot) : \sF \to \sR^{|\mathcal{C}|}$. The {\it margin disparity} from hypothesis of task model $f_{\omega}$ to $h_{\psi}\circ f^1_{\omega}$ is used as the score-based function to measure whether a data instance $\rvx$ lies within the source domain distribution
\begin{definition}[Margin Disparity Discrepancy~\cite{zhang2019bridging}] \label{def:mdd}
    The margin disparity discrepancy is defined as a $\gH$-divergence between source and target domains. 
    \begin{equation}
    d^{\rho}_{f,\gH}(S,T) \triangleq \sup_{f'\in\gH} (disp^{\rho}_S(f',f)-disp^{\rho}_T(f',f))
    \end{equation}
    where $disp^{\rho}_D(f',f)$ is defined as the margin disparity between $f$ and $f'$ in domain $D$
    \begin{equation}
    disp^{\rho}_D(f,f') \triangleq \mathbb{E}_{\mathbf{x}_i, \sim D}\Phi^{\rho}(\rho_{f'}(\rvx_i, h_f(\rvx_i)))
    \end{equation}
    where $\rho_{f} $, $h_f(\rvx_i)$ and $\Phi^{\rho}$ is defined as 
    \begin{equation}
    \rho_{f}(\rvx_i, c) \triangleq f(\mathbf{x}_i, c) - \max_{c'\neq c }f(\mathbf{x}_i, c')
    \end{equation}
    \begin{equation}
    h_f(\rvx_i) \triangleq \argmax_cf(\rvx_i, c)
    \end{equation}
    \begin{equation}
    \Phi^{\rho}(x) =
  \begin{cases}
    1       & \quad \text{if } x < 0\\
    1- x/\rho & \quad \text{if } 0 \leq x \leq \rho \\
    0  & \quad \text{if } x > \rho
  \end{cases}
    \end{equation}
\end{definition}
Again, with the commonly used cross-entropy loss in training objectives, $\gH$-divergence of $d^{\rho}_{f,\gH}(S,T)$ is approximated as 
\begin{equation}
\begin{split}
d^{\rho}_{f,\gH}(S,T) &\approx \mathbb{E}_{S}\log(\softmax(h_\psi(f^1_{\omega}(\mathbf{x}^s_i)), h_f(\mathbf{x}^s_i))) \\
&-\mathbb{E}_{T}\log(1-\softmax(h_\psi(f^1_{\omega}(\mathbf{x}^t_i)),h_f(\mathbf{x}^s_i)))
\end{split}
\end{equation}
The MDD-induced training objective for source-only domain discriminator $\psi_s$ in Equation (\ref{sourcediscobj}) results in
\begin{equation}
\min_{\psi_s}\E_{\rvx_i^s \sim S_0}-\log(\softmax(h_{\psi,s}(f^1_{\omega}(\mathbf{x}^s_i)), \argmax_cf(\mathbf{x}^s_i)))
\end{equation}

The MDD form of adversarial loss for feature extractor $f_{\omega}^1$ from the ensembles of source and target domain discriminator, as expressed in Equation (\ref{fomegaobj}), is given by:
\begin{equation}
\label{advomega}
\begin{split}
&D_{\psi}^s(\mathbf{x}^s_i) = -\log(\softmax(h_{\psi, s}(f^1_{\omega}(\mathbf{x}^s_i))+h_{\psi, t}(f^1_{\omega}(\mathbf{x}^s_i))\\
&, \argmax_cf(\mathbf{x}^s_i)))\\ &D_{\psi}^t(\mathbf{x}^t_i) =- \log(1- \softmax(h_{\psi, s}(f^1_{\omega}(\mathbf{x}^t_i))+h_{\psi, t}(f^1_{\omega}(\mathbf{x}^t_i))\\
&, \argmax_cf(\mathbf{x}^t_i))) 
\end{split}
\end{equation}

The full description of our double-head domain discriminator algorithm for continual UDA is shown in Algorithm \ref{alg:doubleheadalg} of our appendix.

\vspace{-1em}
\section{Theoretical Analysis of Algorithm} \label{sec:theory}
In this section, we relate the source-only domain discriminator $h_{\psi, s}(\cdot)$, which is trained on source domain data and freezed during $T_1$, to a fixed hypothesis $f_0$. Thus we study its effect on target adaptation in $T_1$. First, we show that in the population form, our domain adversarial function from the ensembles of two discriminator's digit constructs a $\gH$-divergence as generalization upper bound between source and target domain task's population risks. 
\begin{theorem}
For a hypothesis class $\gF$ and a fixed $f_0\in\gF$ where for every $f \in \gF$, $f-f_0$ is also in $\gF$, then we have the following property holds
\begin{equation}
err_T(f) \leq err_S^{(\rho)}(f) + d^{(\rho)}_{f,f_0, \gF}(S,T) + \lambda
\end{equation}
where $err_S^{(\rho)}(f)$, $d^{(\rho)}_{f,f_0, \gF}(S,T)$ and $\lambda$ is defined as
\begin{equation}
\begin{split}
&err_S^{(\rho)}(f)= \E_{(x_i,y_i)\sim S}\Phi_{\rho}\circ\rho_{f}(x_i, y_i)\\
&d^{(\rho)}_{f,f_0, \gF}(S,T)=\sup_{f'\in\gF}\{
\E_{x_i\sim T}\Phi_{\rho}\circ\rho_{f'+f_0}(x_i, h_f(x_i))\\
&-\E_{x_i\sim S}\Phi_{\rho}\circ\rho_{f'+f_0}(x_i, h_f(x_i))\}\\
&\lambda=\min_{f^{\star}\in\gF}err_S^{(\rho)}(f^{\star})+err_T^{(\rho)}(f^{\star}),
\end{split}
\end{equation}
\end{theorem}
\vspace{-1em}
\textbf{Remark}~~ The upper bound above has a similar form with the learning bound proposed by \cite{zhang2019bridging}. From the perspective of population loss, 
 our domain loss function from the ensembles of two discriminator's digit is equivalent to that of traditional MDD version where source-only domain discriminator $f_0$ is not introduced. 

Next, we bound on the gap between empirical estimations of domain adversarial loss and its populated version. We first introduce Rademacher complexity as the richness of mapping from an arbitrary input space $\gX \in \sD \to \sR$. 
The following
states the formal definitions of the empirical and average Rademacher complexity.
\begin{definition}
\label{rademacher}
\textbf{(Rademacher Complexity)} Let $\gG$ be a family of functions mapping from $\gX \in \sD \to \sR$. And $\hat{D}=\{(\rvx_0, \rvx_1, \ldots ,\rvx_n)\}$ is a fixed sample of size n drawn from distribution $\gD$ over $\sD$. Then the empirical Rademacher complexity w.r.t sample $\hat{D}$ is defined as
\vspace{-1em}
\begin{equation}
\hat{\Re}_{n, \hat{D}}(\gG)=\E_{\delta}\sup_{g\in\gG}\frac{1}{n}\sum_{i=1}^{n}\delta_ig(\rvx_i)
\end{equation}

where $\delta_i$'s independent uniform random variables taking values $\{+1, -1\}$. The random variables $\delta_i$ are called Rademacher variables.

\noindent The Rademacher complexity
of $\gG$ is the expectation of the empirical Rademacher complexity over all samples of
size n drawn according to $\gD$:
\vspace{-1em}
\begin{equation}
{\Re}_{n, \gD}(\gG) := \E_{\hat{D} \sim\gD}[\hat{\Re}_{n, \hat{D}}(\gG)] 
\end{equation}
\end{definition}
In the following, we define $\gG_s$ as a family of source domain discrepancy loss function associated to $\gF$ mapping from $\gX \in \sD \to \sR$, $\gG_t$ as a family of target domain discrepancy loss function associated to $\gF$ mapping from $\gX$ to $\sR$:
\vspace{-1em}
\begin{equation}
\begin{split}
\gG_s &= \{g_s : x \to \log(\frac{e^{\rho_{f'}(\rvx, h_f)}}{1+e^{\rho_{f'}(\rvx, h_f)}}): f,f' \in \gF \} \\
\gG_t &= \{g_t : x \to \log(\frac{1}{1+e^{\rho_{f'}(\rvx, h_f)}}): f,f' \in \gF \}
\end{split}
\end{equation}
With the Rademacher complexity defined above, we would proceed to show that our $\gH$-divergence based domain adversarial loss could be empirically estimated through finite samples of source domain data and target domain data. 
\begin{theorem}
\label{rademacherbondtheorem}
Let $f_0 \in \gF$ be a fixed hypothesis that maps from $\gX \times \gY \to \sR$ which satisfies $\rho_{f_0}(\rvx^s, h_f)\geq \epsilon_s$ for source domain data $\rvx^s \in S$ and $\rho_{f_0}(\rvx^t, h_f) \leq \epsilon_t$ for target domain data $\rvx^t \in T$. $\rvx_i^s$ is an i.i.d sample of size m drawn from the source distribution S and $\rvx_t^s$ is an i.i.d sample of size n drawn from the target distribution T. Given the same settings as Definition \ref{rademacher}. For any $\delta>0$, with the probability at least $1-2\delta$, we have the following generalization error bound for $\gH$-divergence based adversarial loss function 
\begin{equation}
\begin{split}
&\E_{\rvx^s\in S}[\log(\frac{e^{\rho_{f'}(\rvx^s, h_f)+ \rho_{f_0}(\rvx^s, h_f)}}{1+e^{\rho_{f'}(\rvx^s, h_f)+\rho_{f'}(\rvx^s, h_f)}})] \\
&+ \E_{\rvx^t\in T}[\log(\frac{1}{1+e^{\rho_{f'}(\rvx^t, h_f)+\rho_{f_0}(\rvx^t, h_f)}})]\\
&\leq  \frac{1}{m}\sum_{i=1}^{m}\log(\frac{e^{\rho_{f'}(\rvx^s_i, h_f)+ \rho_{f_0}(\rvx^s_i, h_f)}}{1+e^{\rho_{f'}(\rvx^s_i, h_f)+\rho_{f'}(\rvx^s_i, h_f)}})\\
&+ \frac{1}{n}\sum_{i=1}^{n}\log(\frac{1}{1+e^{\rho_{f'}(\rvx^t_i, h_f)+\rho_{f'}(\rvx^t_i, h_f)}})\\
&+ \max\{\frac{2}{(e^{\epsilon_s}-1)\lambda_s^++1}, \frac{2}{(e^{\epsilon_s}-1)\lambda_s^-+1}\}{\Re}_{m, \gD_s}(\gG_s) \\
&+ \max\{\frac{2e^{\epsilon_t}}{(1-\lambda_t^+)e^{\epsilon_t}+\lambda_t^+},\frac{2e^{\epsilon_t}}{(1-\lambda_t^-)e^{\epsilon_t}+\lambda_t^-}\}{\Re}_{n, \gD_t}(\gG_t) \\
&+ \sqrt{\frac{\log\frac{1}{\delta}}{2m}} + \sqrt{\frac{\log\frac{1}{\delta}}{2n}}
\end{split}
\end{equation}
where $\lambda_s^+$, $\lambda_s^-$, $\lambda_t^+$  and $\lambda_t^-$ is defined as
\begin{equation}
\begin{split}
\lambda_s^- &= \min\{\frac{e^{\rho_{f'}(\rvx^s, h_f)}}{1+e^{\rho_{f'}(\rvx^s, h_f)}}\},  \lambda_s^+ = \max\{\frac{e^{\rho_{f'}(\rvx^s, h_f)}}{1+e^{\rho_{f'}(\rvx^s, h_f)}}\}, \forall \rvx^s \in S  \\ 
\lambda_t^- &= \min\{\frac{1}{1+e^{\rho_{f'}(\rvx^t, h_f)}}\}, \lambda_t^+ = \max\{\frac{1}{1+e^{\rho_{f'}(\rvx^t, h_f)}}\}, \forall \rvx^t \in T 
\end{split}
\end{equation}
\end{theorem}


\textbf{Remark}~~ This theorem justifies that the populated domain adversarial loss with respected to $\gH$-divergence could be approximated by the empirical one computed from finite source and target domain samples. With the introduction of $f_0$, the empirical generalization error from source domain side is reduced with a larger source domain source of $\epsilon_s$. Conversely over-training on $f_0$ would cause a larger $\epsilon_t$ which would result in a larger generalization error from target domain side. Our theorem also emphasises the importance of training a better one-class score based function $\rho_{f_0}(\rvx^s, h_f)$ with a higher score for in-distribution data on source domains than outliers. 

Finally, we analyze on the equilibrium of our adversarial loss w.r.t generator and discriminators. We would show that how our introduced source only domain discriminator's score $\sigmoid_{h_f}\circ f'_s$ controls the magnitudes of consistency between source and target domain's distributions.
\begin{proposition}
Consider the following optimization problem we have defined
\begin{equation}
\begin{split}
&\max_{f'}\E_{\hat{S}}\log(\sigmoid_{h_f}\circ f')+\E_{\hat{T}}\log(1-\sigmoid_{h_f}\circ f')\\
&\min_{\hat{S},\hat{T}}\E_{\hat{S}}\log(\frac{1}{2}\sigmoid_{h_f}\circ f'+\frac{1}{2}\sigmoid_{h_f}\circ f_0)\\
&+\E_{\hat{T}}\log(1-\frac{1}{2}\sigmoid_{h_f}\circ f'-\frac{1}{2}\sigmoid_{h_f}\circ f_0)
\end{split}
\end{equation}
Assume that there is no restriction on the choice of $f'$. By fixing $f_0$, we have the following result.

The minimization problem w.r.t $S$ and $T$ is equivalent to minimization on the sum of two terms $L_1$ and $L_2$
\begin{equation}
\begin{split}
    &L_1 =\\
    &4KL(\frac{3}{4}\hat{T}+\frac{1}{4}\hat{S}||\frac{1}{2}\hat{T}+\frac{1}{2}\hat{S}) + 4KL(\frac{1}{2}\hat{T}+\frac{1}{2}\hat{S}||\frac{3}{4}\hat{T}+\frac{1}{4}\hat{S})\\
     +& 4KL(\frac{3}{4}\hat{S}+\frac{1}{4}\hat{T}||\frac{1}{2}\hat{T}+\frac{1}{2}\hat{S}) + 4KL(\frac{1}{2}\hat{T}+\frac{1}{2}\hat{S}||\frac{3}{4}\hat{S}+\frac{1}{4}\hat{T})
\end{split}
\end{equation}
is a symmetric distribution divergence between $\hat{S}$ and $\hat{T}$ and has global minimum at $\hat{S}=\hat{T}$
\begin{equation}
\begin{split}
L_2 =&\int_{\rvx}(1-2\sigmoid_{h_f}\circ f_0(\rvx))(\hat{q}_t(\rvx)-\hat{p}_s(\rvx))\\
&\frac{1}{4-(\hat{p}_s(\rvx)-\hat{q}_t(\rvx))^2/(\hat{p}_s(\rvx)+\hat{q}_t(\rvx))^2}d\rvx
\end{split}
\end{equation}
is a re-weighted bounds on the total variations between $\hat{p}_s(\rvx)$ and $\hat{q}_t(\rvx)$
\end{proposition}

\textbf{Remark}~~ 
Recall that $\sigmoid_{h_f}\circ f_0(\rvx)$ is the output score of source only domain discriminator for the possibilities of $\rvx$ belongs to source domain. Assuming that for in-distribution area of source domain $\rvx$, $\sigmoid_{h_f}\circ f_{0}(\rvx)>\epsilon$, where $\hat{p}_s(\rvx)-\hat{q}_t(\rvx)>0$. Otherwise for out-of-distribution area $\sigmoid_{h_f}\circ, f_{0}(\rvx)<\epsilon$, where $\hat{p}_s(\rvx)-\hat{q}_t(\rvx)>0$. $L_2$ would be further approximated as $\Tilde{L}_2$
\begin{equation}
\begin{split}
&\Tilde{L}_2 = 2\int_{\rvx}\ (\sigmoid_{h_f}\circ f_0(\rvx)-\epsilon)(\hat{q}_t(\rvx)-\hat{p}_s(\rvx))\\
&\frac{1}{4-(\hat{p}_s(\rvx)-\hat{q}_t(\rvx))^2/(\hat{p}_s(\rvx)+\hat{q}_t(\rvx))^2}d\rvx\\
&\|\Tilde{L}_2 - L_2 \| \leq \frac{1}{12}\|1-2\epsilon\|
\end{split}
\end{equation}
Furthermore, since $\sigmoid_{h_f}\circ f_{0}(\rvx)$ is learned on the entire source domain dataset as a source-based function that relates to source domain's distribution density, it re-weights the empirical distribution $\hat{p}_s(\rvx)$ based on a small number of samples from the source domain stored in $\gM$.  
\section{Related Works}
\textbf{Unsupervised Domain Adaptation} For UDA methods, besides adversarial domain adaptation\cite{ganin2016domain,zhao2018adversarial, long2017conditional,zhang2019bridging, saito2018maximum} that learns feature representations invariant between source and target domain. Self-Training(ST)\cite{arazo2020pseudo, pham2021meta} and Knowledge Distillation\cite{liang2022dine} is also widely adopted for UDA. Other works\cite{ding2022source, fleuret2021uncertainty, dey2022leveraging, kundu2020universal, li2020model, tian2021vdm, wang2022exploring, xia2021adaptive, yang2021generalized,yeh2021sofa, gong2022robust, niu2022efficient} also address the problem of UDA. Notably 
However these works require either freezing on the task model trained on source domain, caching the prototypical features of source domain or demanding specific engineering on task model structure, which limits its application in the general setting of CL. Notably, federated UDA\cite{shen2023fedmm, peng2019federated} proposes a simplified version CL where all labeled datasets are simultaneously accessible in a spatial isolated case.

\textbf{Domain Incremental Learning} The main goal for domain incremental learning is to consistently learn information on new domain, without forgetting the knowledge of previous domains. The first category of methods is by incrementally adding new task heads to fit on new domains\cite{rusu2016progressive, zhou2012online,cortes2017adanet, yoonlifelong, ji2023continual}. The second category of methods is using memory replay methods to store the data of previous domains\cite{lopez2017gradient,chaudhryefficient, chaudhry2019continual,riemerlearning,hayes2020remind, prabhu2020gdumb}. The third category of methods is to add regularization terms to constraint task's objectives to avoid forgetting\cite{li2017learning, kirkpatrick2017overcoming, zenke2017continual,fini2020online,volpi2021continual, ma2022progressive}.  

\section{Experiments} \label{sec:experiments}
To evaluate the effectiveness of double-head discriminator algorithm for Continual UDA. We first describe the benchmark datasets and other experiment settings in Section \ref{expsetup} of our appendix. Then we perform ablation study in Section \ref{ablsty}. Next, we compare ours with various existing methods in \ref{compuda}.
\subsection{Ablation Study}
\label{ablsty}

\noindent\textbf{Effect of Different Memory Size}
To investigate the effect of different memory sizes on the model performance, we evaluate on the task of continual adaptation to MNISTM, USPS and SVHN with memory sizes of 8, 16, 32, 64, 128 on each class of source domain(MNIST). The $+\infty$ shows the cases of offline adaptation where all source and target data would be accessed in an i.i.d way. We show our result in Fig (\ref{fig:memsize}). With the increasing memory buffer size, the performance would slightly increase. However our algorithm entails minimal performance loss from the smaller memory buffer size.

\begin{figure}
\centering
\begin{subfigure}{.15\textwidth}
  \centering
  \includegraphics[width=1.0\linewidth]{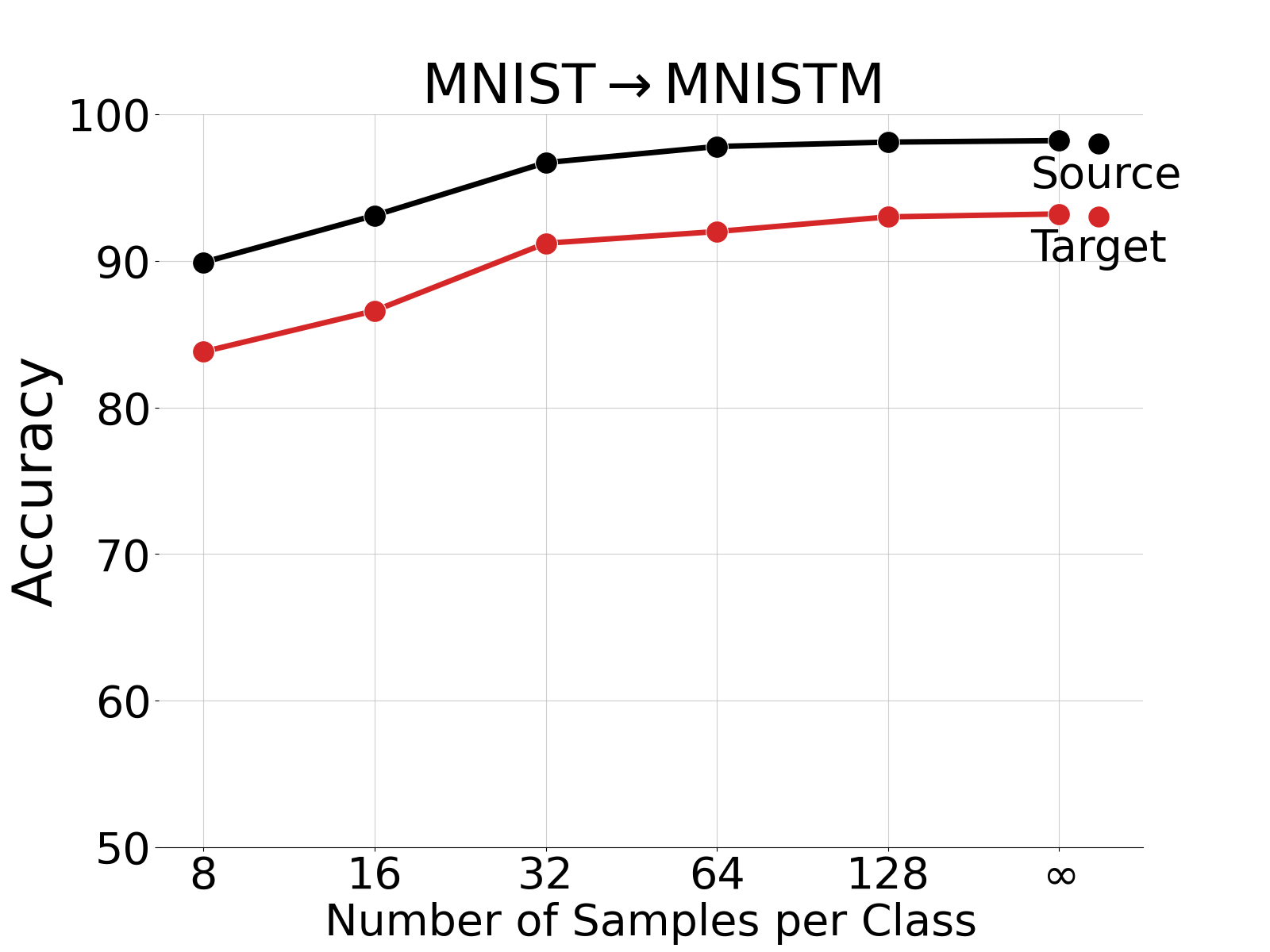}
\end{subfigure}%
\begin{subfigure}{.16\textwidth}
  \centering
  \includegraphics[width=1.0\linewidth]{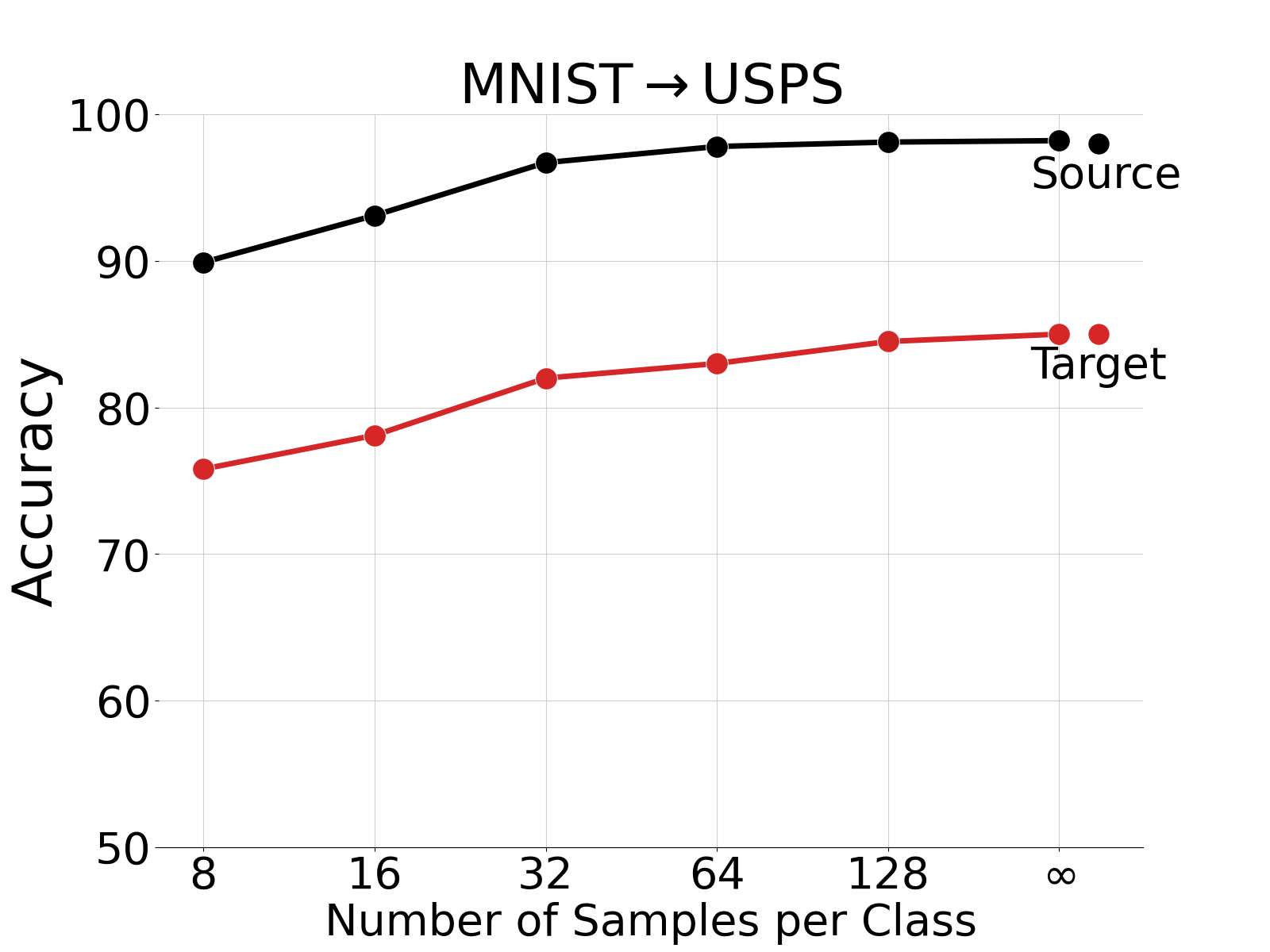}
\end{subfigure}
\begin{subfigure}{.16\textwidth}
  \centering
  \includegraphics[width=1.0\linewidth]{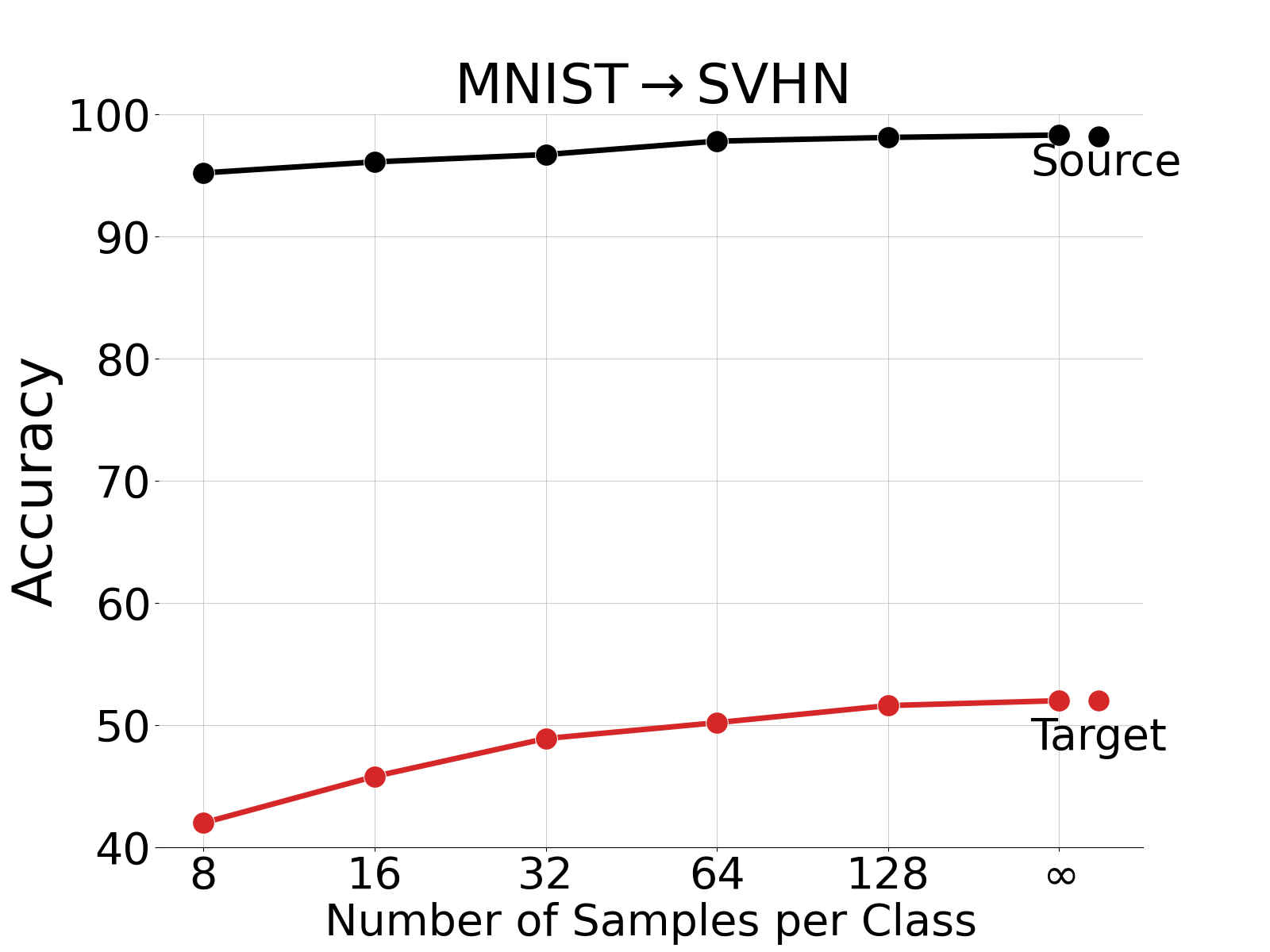}
\end{subfigure}%
\caption{Effect of different memory size on model performance}
\label{fig:memsize}
\end{figure}

\noindent\textbf{The Benefit of Source only Domain Discriminator}
To investigate the necessary of introducing source only domain discriminator $h_{\psi,s}$ in phase $T_1$ , we evaluate the contribution of $h_{\psi, s}$'s digits on learning the task model $f_{\omega}$ to adapt on target domain. 
Specifically, we use $h_{\psi,s}(f_{\omega}^1(\rvx_i)) + \gamma h_{\psi,t}((f_{\omega}^1(\rvx_i)))$ in Equation (\ref{advomega}) as the domain discriminator's signal to adapt $f_{\omega}$. We shift $\gamma$ from $0$ to $1$ where $h_{\psi, s}$ is gradually mixed with $h_{\psi, t}$. The result in Fig (\ref{fig:sdc}) showed that the performance is significantly lower in $\gamma=0$ where source only domain discriminator is not used for adaptation in phase $T_1$. Our result emphasizes the importance of introducing an additional pre-trained domain discriminator on $S_0$ phase. The choice of $\gamma$ to ensemble domain predictions from $h_{\psi,t}$ and $h_{\psi,s}$ adopts a wide range from $0.2$ to $1$. In MNIST to SVHN task, $\gamma=0.2$ has better result because the data variations of SVHN is much larger than MNIST and a smaller $\gamma$ would have less empirical error from target domain side as we have analyzed on Theorem \ref{rademacherbondtheorem}.

\begin{figure}
\centering
\begin{subfigure}{.15\textwidth}
  \centering
  \includegraphics[width=1.\linewidth]{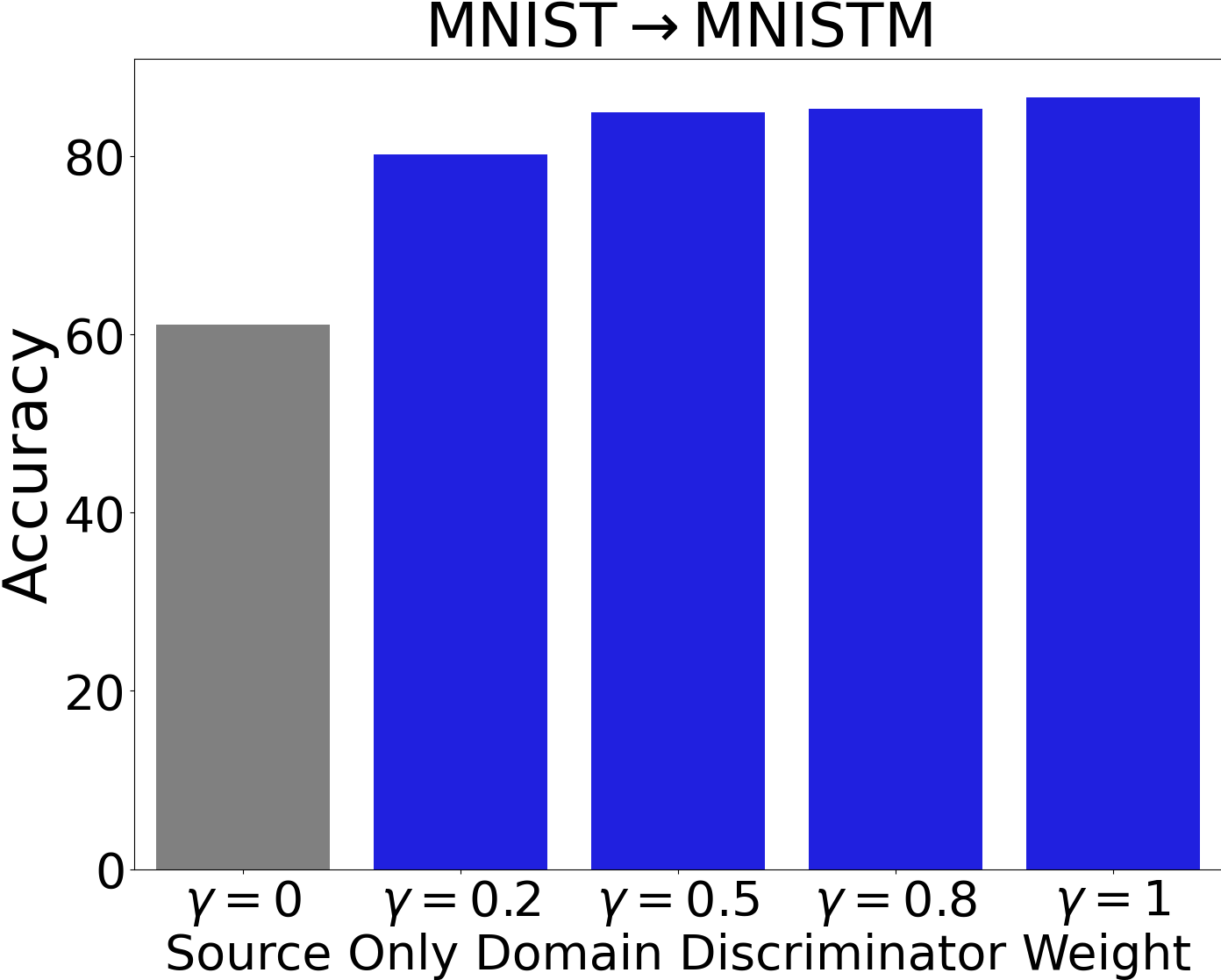}
\end{subfigure}%
\begin{subfigure}{.15\textwidth}
  \centering
  \includegraphics[width=1.\linewidth]{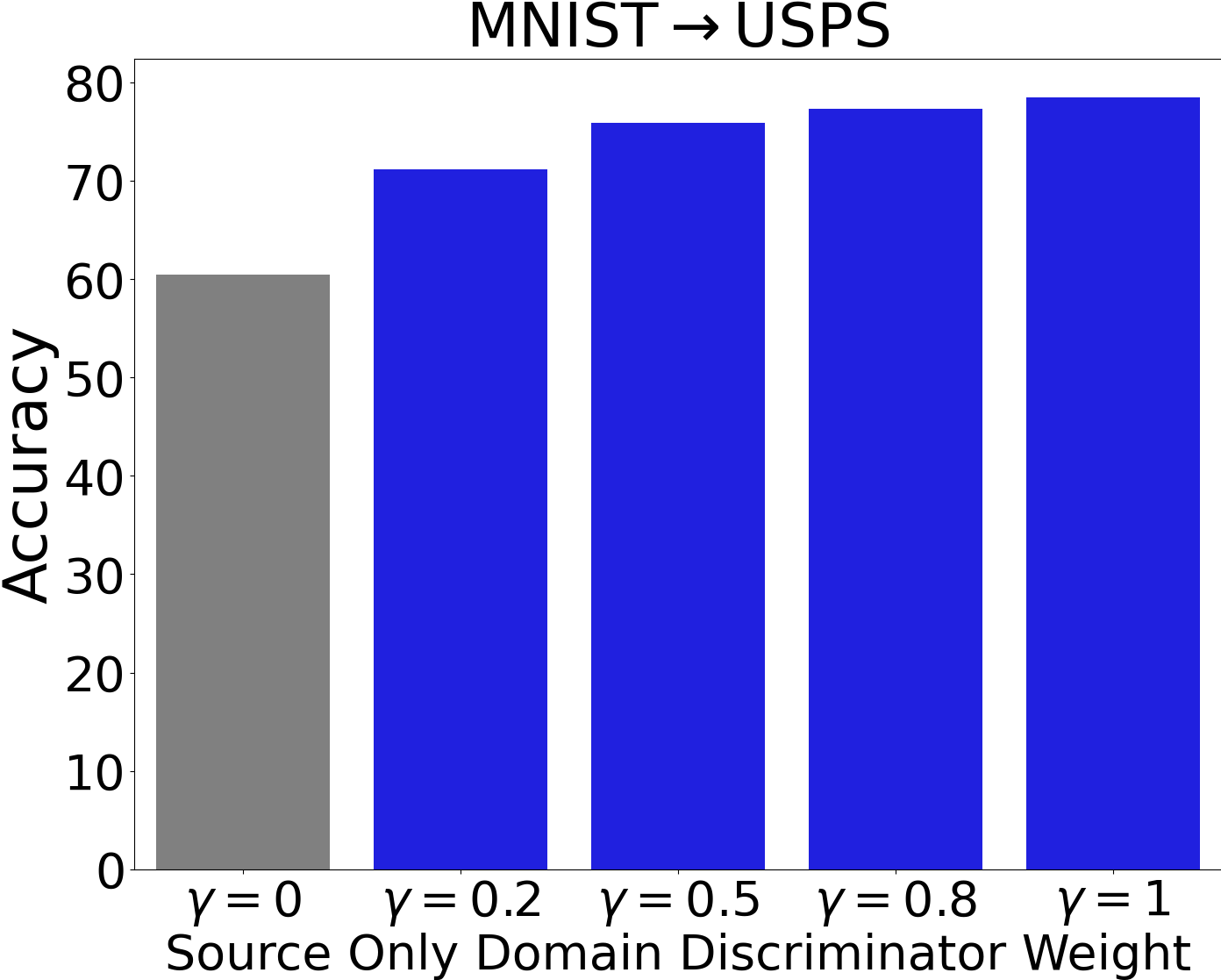}
\end{subfigure}
\begin{subfigure}{.15\textwidth}
  \centering
  \includegraphics[width=1.\linewidth]{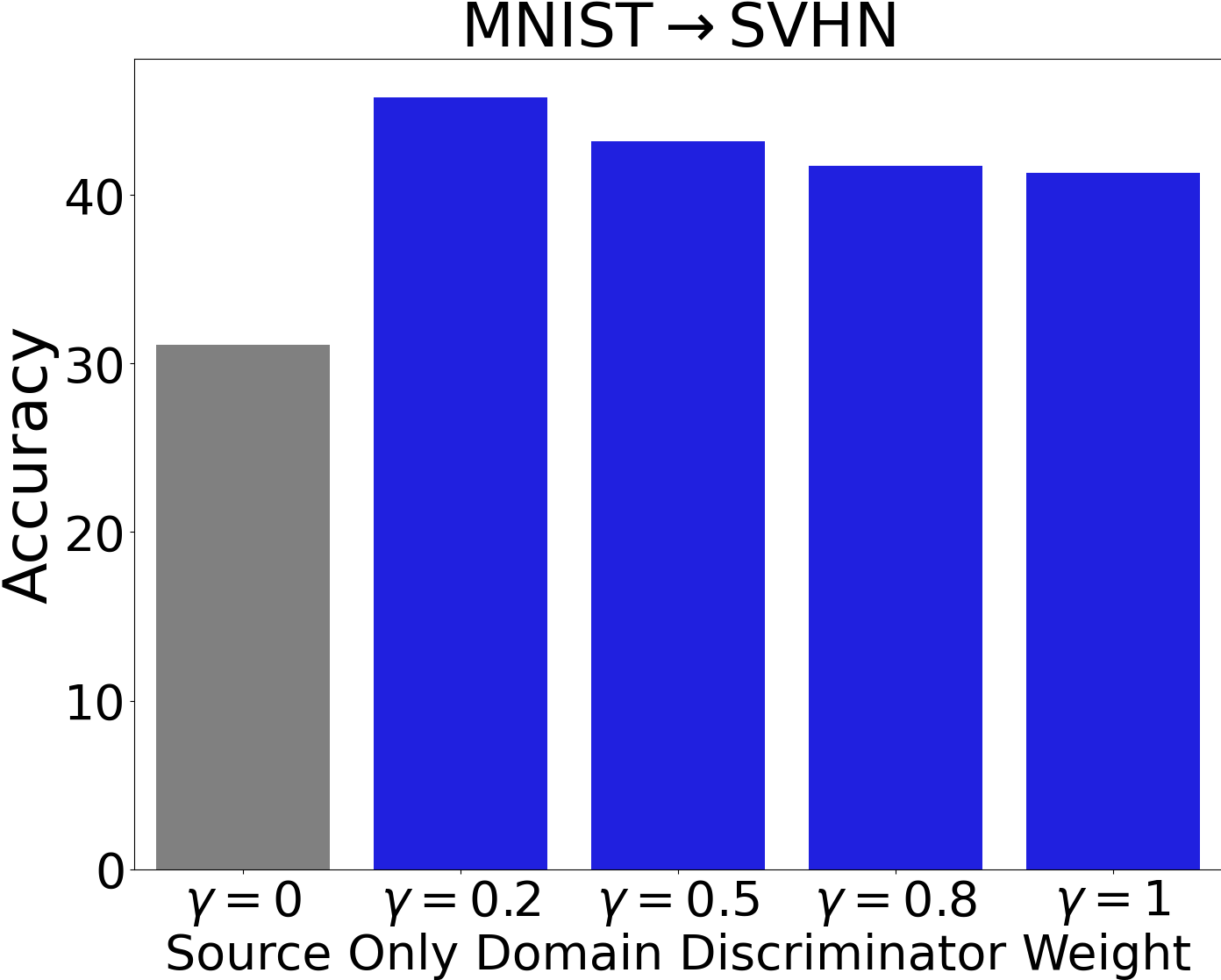}
\end{subfigure}%
\caption{Effect of source only domain discriminator's contribution on target adaptation performance}
\label{fig:sdc}
\end{figure}

\begin{table}
\begin{adjustbox}{width=0.5\textwidth}
  \begin{tabular}{@{}lr|r|r|r|r|r@{\hskip 0.1in}l@{}}
  \toprule
                                  &     \multicolumn{1}{c}{}                & \multicolumn{4}{c}{Office-31 Target Domain Adaptations}                                                                                                                                                                                                                                                                                                                                                                                                                                                      &  \\ \cmidrule(l){2-7} 
                          Methods      & \multicolumn{1}{c|}{$A\to W$}                                                                                        & \multicolumn{1}{c|}{$D\to W$}                                                                                       & \multicolumn{1}{c|}{$W\to D$}                                                                                                             & \multicolumn{1}{c|}{$A\to D$}               & \multicolumn{1}{c|}{$D\to A$}
                                 & \multicolumn{1}{c}{$W\to A$} &  \\ \midrule              

NLL-OT{\cite{asano2019self}}     &  85.5 & 95.1 &  98.7 &  88.8         &  64.6         & 66.7         &  \\ \midrule
NLL-KL{\cite{zhang2021unsupervised}}      &  86.8         & 94.8 & 98.7                  & 89.4           & 65.1                   & 67.1 &   \\ \midrule
HD-SHOT{\cite{liang2020we}}                & 83.1          & 95.1 & 98.1          &  86.5          &  66.1          &  68.9 &   \\ \midrule
SD-SHOT{\cite{liang2020we}}      &  83.7        &  95.3&  97.1          & 89.2          & 67.9 & 71.1                  &   \\ \midrule
DINE{\cite{liang2022dine}}      &  86.8          &  96.2 &  98.6         &  91.6          & 72.2      &     73.3 &   

\\ \midrule

{\cellcolor[HTML]{EFEFEF}Ours}              & \cellcolor[HTML]{EFEFEF}92.6           & \cellcolor[HTML]{EFEFEF}97.3                & \cellcolor[HTML]{EFEFEF}99.2        & \cellcolor[HTML]{EFEFEF}92.0        & \cellcolor[HTML]{EFEFEF}73.9           & \cellcolor[HTML]{EFEFEF}73.8 &\\
\midrule
{\cellcolor[HTML]{EFEFEF}Ours+KD}              & \cellcolor[HTML]{EFEFEF}\textbf{93.8}           & \cellcolor[HTML]{EFEFEF}\textbf{98.4}                & \cellcolor[HTML]{EFEFEF}\textbf{100.0}        & \cellcolor[HTML]{EFEFEF}\textbf{93.8}        & \cellcolor[HTML]{EFEFEF}\textbf{74.0}           & \cellcolor[HTML]{EFEFEF}\textbf{75.6} &\\
\midrule
{\cellcolor[HTML]{EFEFEF}Ours+SL}              & \cellcolor[HTML]{EFEFEF}93.2           & \cellcolor[HTML]{EFEFEF}97.7                & \cellcolor[HTML]{EFEFEF}100.0        & \cellcolor[HTML]{EFEFEF}92.5        & \cellcolor[HTML]{EFEFEF}73.9           & \cellcolor[HTML]{EFEFEF}74.4 &\\
\midrule
{\cellcolor[HTML]{fde8e7}i.i.d-adv}              & \cellcolor[HTML]{fde8e7}94.5           & \cellcolor[HTML]{fde8e7}98.4                & \cellcolor[HTML]{fde8e7}100.0        & \cellcolor[HTML]{fde8e7}93.5        & \cellcolor[HTML]{fde8e7}74.6          & \cellcolor[HTML]{fde8e7}74.2
&  \\ \bottomrule
  \end{tabular}%
\end{adjustbox}
    \caption{Office-31 Target Domain Adaptation}
    \label{tab:office31taskadp}
\end{table}

\begin{table}
\begin{adjustbox}{width=0.5\textwidth}
  \begin{tabular}{@{}lr|r|r|r|r|r@{\hskip 0.1in}l@{}}
  \toprule
                                 &     \multicolumn{1}{c}{}                & \multicolumn{4}{c}{Office-31 Source Domain Forgetting}                                                                                                                                                                                                                                                                                                                                                                                                                                                      &  \\ \cmidrule(l){2-7} 
                          Methods      & \multicolumn{1}{c|}{$A\to W$}                                                                                        & \multicolumn{1}{c|}{$D\to W$}                                                                                       & \multicolumn{1}{c|}{$W\to D$}                                                                                                             & \multicolumn{1}{c|}{$A\to D$}               & \multicolumn{1}{c|}{$D\to A$}
                                 & \multicolumn{1}{c}{$W\to A$} &  \\ \midrule              

NLL-OT{\cite{asano2019self}}     &  4.53 & 3.14 &  2.73 &  4.30         &  6.17          & 5.11         &  \\ \midrule
NLL-KL{\cite{zhang2021unsupervised}}      &  4.37          & 2.99 & 2.48                   & 4.02           & 5.94                   & 4.99 &   \\ \midrule
HD-SHOT{\cite{liang2020we}}                & 5.12          & 4.01 & 3.98          &  4.87          &  7.80          &  5.56 &   \\ \midrule
SD-SHOT{\cite{liang2020we}}      &  5.31         &  4.54 &  4.03          & 4.85          & 7.88 &        5.72            &   \\ \midrule
DINE{\cite{liang2022dine}}     &    3.81        &  2.16 &  1.50         &  3.32          & 5.08     &     3.98 &   

\\ \midrule

{\cellcolor[HTML]{EFEFEF}Ours}              & \cellcolor[HTML]{EFEFEF}$\textbf{1.97}$           & \cellcolor[HTML]{EFEFEF}\textbf{1.03}                & \cellcolor[HTML]{EFEFEF}\textbf{0.98}        & \cellcolor[HTML]{EFEFEF}\textbf{1.55}        & \cellcolor[HTML]{EFEFEF}\textbf{3.72}           & \cellcolor[HTML]{EFEFEF}\textbf{2.96}  &\\ \bottomrule
  \end{tabular}%
\end{adjustbox}
  \caption{Office-31 Source Domain Forgetting}
    \label{tab:office31forget}
\end{table}
\noindent\textbf{Effect of Learning rate and Epoch on Source-only Domain Discriminator}
Training source only domain discriminator has resemblance of one-class learning on a score-based function, the learning rate and epoch plays a important role to assure sensitivity to in-distribution data while prevent over saturation. Here we investigate the effect of learning rate and epoch on source-only domain discriminator in this section. We plot the combinatorial case on the learning rate of 0.0001, 0.0004, 0.001 and 0.002 and epochs of 1, 3, 5 and 7 training epochs as a heatmap that are shown in Fig (\ref{fig:heatmap}). We observe that pre-training a source-only domain discriminator in $S_0$ phase with a smaller learning rate and moderate number of training epochs would lead to better performance of target adaption in $T_1$ phase. This accords with the observation of the learning rate and epoch's effect on the performance of one-class learning in \cite{hu2020hrn}. For a stable and optimized performance, we choose the learning epoch source of 5 with learning rate of 0.0001 in the rest of our experiment. 

\begin{figure*}[htb!]
\centering
\begin{subfigure}{.33\textwidth}
  \centering
  \includegraphics[width=1.0\linewidth]{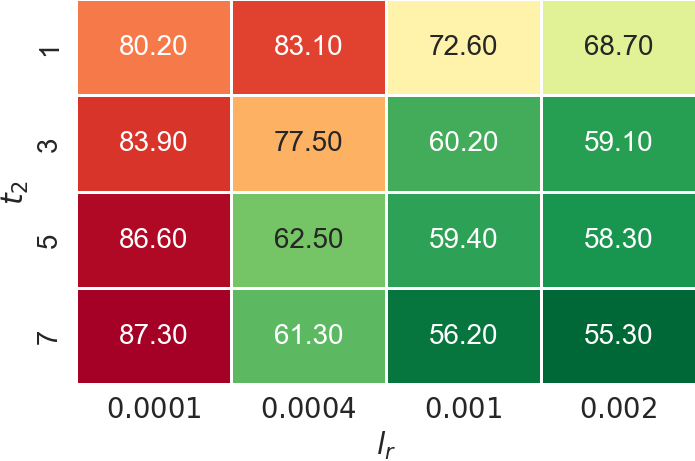}
  \caption{MNIST$\to$MNISTM}
\end{subfigure}%
\begin{subfigure}{.33\textwidth}
  \centering
  \includegraphics[width=1.0\linewidth]{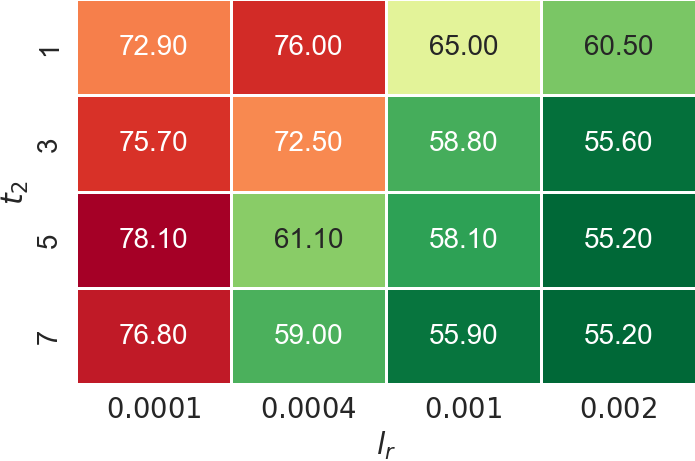}
  \caption{MNIST$\to$USPS}
\end{subfigure}
\begin{subfigure}{.33\textwidth}
  \centering
  \includegraphics[width=1.0\linewidth]{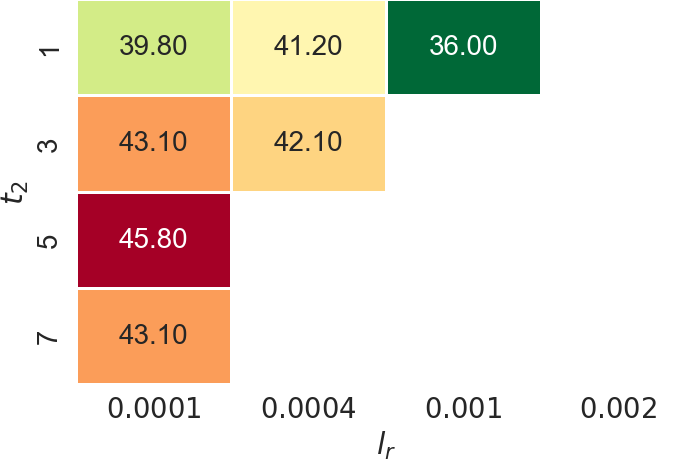}
  \caption{MNIST$\to$SVHN}
\end{subfigure}%
\caption{Effect of Source-only Domain Discriminator's learning rate $l_r$ and training epochs $t_2$ on target adaptation performance.}
\label{fig:heatmap}
\end{figure*}

\begin{table*}[!htp]
  \centering  

 \begin{adjustbox}{width=0.95\textwidth}
  \begin{tabular}{@{}lr|r|r|r|r|r|r|r|r|r|r|r@{\hskip 0.1in}l@{}}
  \toprule
                                 &       \multicolumn{1}{c}{}           & \multicolumn{9}{c}{Office-home Target Domain Adaptations}                                                                                                                                                                                                                                                                                                                                                                                                                                                      &  \multicolumn{1}{c}{}   \\ \cmidrule(l){2-13} 
                          Methods      & \multicolumn{1}{c|}{$Ar\to Cl$}                                                                                        & \multicolumn{1}{c|}{$Ar\to Pr$}                                                                                       & \multicolumn{1}{c|}{$Ar\to Re$}                                                                                                             & \multicolumn{1}{c|}{$Cl\to Ar$}               & \multicolumn{1}{c|}{$Cl\to Pr$}
                                 & \multicolumn{1}{c|}{$Cl\to Re$} & \multicolumn{1}{c|}{$Pr\to Ar$} &
                                 \multicolumn{1}{c|}{$Pr\to Cl$} &
                                 \multicolumn{1}{c|}{$Pr\to Re$} &
                                 \multicolumn{1}{c|}{$Re\to Ar$} &
                                 \multicolumn{1}{c|}{$Re\to Cl$} &
                                 \multicolumn{1}{c}{$Re\to Pr$} &
                                 \\ \midrule              

NLL-OT{\cite{asano2019self}}     &  49.1 & 71.7 &  77.3 &  60.2         &  68.7          & 73.1    &  57.0 & 46.5 &  76.8 &  67.0         &  52.3         & 79.5           &  \\ \midrule
NLL-KL{\cite{zhang2021unsupervised}}      &  49.0          & 71.5 & 77.1                  & 59.0           & 68.7                   & 72.9 &  56.4 & 46.9 &  76.6 &  66.2         &  52.3          & 79.1        &   \\ \midrule
HD-SHOT{\cite{liang2020we}}               & 48.6          & 72.8 & 77.0          &  60.7         &  70.0          &  73.2 &  56.6 & 47.0 &  76.7 &  67.5         &  52.6          & 80.2      &   \\ \midrule
SD-SHOT{\cite{liang2020we}}       &  50.1         &  75.0 &  78.8          & 63.2          & 72.9 & 76.4                &  60.0 & 48.0 &  79.4 &  69.2         &  54.2          & 81.6        &   \\ \midrule
DINE{\cite{liang2022dine}}       &  52.2          &  78.4 &  81.3         &  65.3          & 76.6      &     78.7 &  62.7 & 49.6 &  82.2 &  69.8         &  55.8          & 84.2       &   

\\ \midrule

{\cellcolor[HTML]{EFEFEF}Ours}              & \cellcolor[HTML]{EFEFEF}53.8           & \cellcolor[HTML]{EFEFEF}78.8                & \cellcolor[HTML]{EFEFEF}81.9       & \cellcolor[HTML]{EFEFEF}66.4        & \cellcolor[HTML]{EFEFEF}77.8           & \cellcolor[HTML]{EFEFEF}77.9 &\cellcolor[HTML]{EFEFEF}63.0          & \cellcolor[HTML]{EFEFEF}52.9                & \cellcolor[HTML]{EFEFEF}83.2        & \cellcolor[HTML]{EFEFEF}72.0        & \cellcolor[HTML]{EFEFEF}59.4           & \cellcolor[HTML]{EFEFEF}84.9 &\\
\midrule
{\cellcolor[HTML]{EFEFEF}Ours+KD}              & \cellcolor[HTML]{EFEFEF}\textbf{54.8}           & \cellcolor[HTML]{EFEFEF}\textbf{81.1}                & \cellcolor[HTML]{EFEFEF}\textbf{84.0}       & \cellcolor[HTML]{EFEFEF}\textbf{67.5}        & \cellcolor[HTML]{EFEFEF} \textbf{79.0}          & \cellcolor[HTML]{EFEFEF}\textbf{80.5} &\cellcolor[HTML]{EFEFEF}\textbf{65.1}           & \cellcolor[HTML]{EFEFEF}\textbf{53.8}                & \cellcolor[HTML]{EFEFEF}\textbf{84.5}        & \cellcolor[HTML]{EFEFEF}\textbf{73.2}        & \cellcolor[HTML]{EFEFEF}\textbf{60.0}           & \cellcolor[HTML]{EFEFEF}\textbf{86.7} &\\
\midrule
{\cellcolor[HTML]{EFEFEF}Ours+SL}              & \cellcolor[HTML]{EFEFEF}54.0           & \cellcolor[HTML]{EFEFEF}79.2                & \cellcolor[HTML]{EFEFEF}82.4      & \cellcolor[HTML]{EFEFEF}66.8        & \cellcolor[HTML]{EFEFEF}78.3           & \cellcolor[HTML]{EFEFEF}79.0 &\cellcolor[HTML]{EFEFEF}63.7           & \cellcolor[HTML]{EFEFEF}53.2                & \cellcolor[HTML]{EFEFEF}83.2        & \cellcolor[HTML]{EFEFEF}72.8        & \cellcolor[HTML]{EFEFEF}59.4           & \cellcolor[HTML]{EFEFEF}85.8 &\\
\midrule
{\cellcolor[HTML]{fde8e7}i.i.d-adv}              & \cellcolor[HTML]{fde8e7}54.9           & \cellcolor[HTML]{fde8e7}79.0                & \cellcolor[HTML]{fde8e7}82.8        & \cellcolor[HTML]{fde8e7}67.0        & \cellcolor[HTML]{fde8e7}78.7           & \cellcolor[HTML]{fde8e7}78.1  & \cellcolor[HTML]{fde8e7}63.6          & \cellcolor[HTML]{fde8e7}54.2                & \cellcolor[HTML]{fde8e7}83.8       & \cellcolor[HTML]{fde8e7}72.9       & \cellcolor[HTML]{fde8e7}60.8          & \cellcolor[HTML]{fde8e7}85.8
&  \\ \bottomrule
  \end{tabular}%
  \end{adjustbox}
  \caption{Comparison of Target Domain Adaptation Performance on Office-home.}
    \label{tab:officehometaskadp}
\end{table*}
\begin{table*}[!htp]
  \centering  
 \begin{adjustbox}{width=0.95\textwidth}
  \begin{tabular}{@{}lr|r|r|r|r|r|r|r|r|r|r|r@{\hskip 0.1in}l@{}}
  \toprule
                                 &       \multicolumn{1}{c}{}           & \multicolumn{9}{c}{Office-home Source Domain Forgetting}                                                                                                                                                                                                                                                                                                                                                                                                                                                      &  \multicolumn{1}{c}{}   \\ \cmidrule(l){2-13} 
                          Methods      & \multicolumn{1}{c|}{$Ar\to Cl$}                                                                                        & \multicolumn{1}{c|}{$Ar\to Pr$}                                                                                       & \multicolumn{1}{c|}{$Ar\to Re$}                                                                                                             & \multicolumn{1}{c|}{$Cl\to Ar$}               & \multicolumn{1}{c|}{$Cl\to Pr$}
                                 & \multicolumn{1}{c|}{$Cl\to Re$} & \multicolumn{1}{c|}{$Pr\to Ar$} &
                                 \multicolumn{1}{c|}{$Pr\to Cl$} &
                                 \multicolumn{1}{c|}{$Pr\to Re$} &
                                 \multicolumn{1}{c|}{$Re\to Ar$} &
                                 \multicolumn{1}{c|}{$Re\to Cl$} &
                                 \multicolumn{1}{c}{$Re\to Pr$} &
                                 \\ \midrule              

NLL-OT{\cite{asano2019self}}     &  10.91 & 7.64 &  7.31 &  12.73         &  13.18          & 11.13    &  7.29 & 7.72 &  6.19 &   7.07        &  7.28          & 5.35           &  \\ \midrule
NLL-KL{\cite{zhang2021unsupervised}}      &  10.93          & 7.66 & 7.34                   & 13.01           & 13.05                   & 10.98 &  7.27 &  7.50 &  6.03 &  6.97         &  7.26          & 5.46        &   \\ \midrule
HD-SHOT{\cite{liang2020we}}                 & 11.10          & 9.69 & 8.06          &  14.99          &  15.02          &  12.06 &  7.57 & 7.86 &  6.58 &  7.22         &  7.92          & 6.02      &   \\ \midrule
SD-SHOT{\cite{liang2020we}}     &  11.21         &  8.93 &  7.89          & 15.24          & 15.55 & 12.25                &  7.75 & 7.93 &  6.72 &  7.22         &  8.13          & 6.05        &   \\ \midrule
DINE{\cite{liang2022dine}}       &  9.67          &  6.66 &   6.26        &  9.29         & 10.02      &     9.76&  6.13 & 5.92 &  5.82 &  6.19         &  6.05          & 4.93       &   

\\ \midrule

{\cellcolor[HTML]{EFEFEF}Ours}              & \cellcolor[HTML]{EFEFEF}\textbf{4.52}           & \cellcolor[HTML]{EFEFEF}\textbf{3.95}                & \cellcolor[HTML]{EFEFEF}\textbf{3.53}        & \cellcolor[HTML]{EFEFEF}\textbf{5.12}        & \cellcolor[HTML]{EFEFEF}\textbf{4.83}           & \cellcolor[HTML]{EFEFEF}\textbf{4.69} &\cellcolor[HTML]{EFEFEF}\textbf{1.93}           & \cellcolor[HTML]{EFEFEF}\textbf{2.05}                & \cellcolor[HTML]{EFEFEF}\textbf{1.89}        & \cellcolor[HTML]{EFEFEF}\textbf{2.12}        & \cellcolor[HTML]{EFEFEF}\textbf{3.13}           & \cellcolor[HTML]{EFEFEF}\textbf{1.43} &\\
\bottomrule
  \end{tabular}%
  \end{adjustbox}
  \caption{Comparison of Source Domain Forgetting Performance on Office-home.}
    \label{tab:officehomeforget}
\end{table*}

\subsection{Comparison to existing Continual UDA}
\label{compuda}
\noindent\textbf{Baseline} We compare our proposed method with two strong baselines, Knowledge Distillation(KD) and Self-Learning(ST), which are commonly used semi-supervised learning(SSL) techniques for continual UDA. KD transfers knowledge from the source to the target domain by distilling on source domain teach model's pseudolikelihoods assigned to unlabeled target domain\cite{hinton2015distilling}. A representative work on of KD is DINE\cite{liang2022dine}. ST trains the model on source labeled data first, then iteratively assigns pseudo-labels to unlabeled target domain and trains on the most confident predictions\cite{nigam2000analyzing}. Variants of ST include NLL-OT\cite{asano2019self} and NLL-KL\cite{zhang2021unsupervised}. SHOT\cite{liang2020we} combines ST with K-Means by using a strategy of assigning pseudo-label from its distance to cluster centroid. 

 Results with office-31 are presented in Table (\ref{tab:office31taskadp}, \ref{tab:office31forget}), and those for office-home are shown in Table (\ref{tab:officehometaskadp}, \ref{tab:officehomeforget}).  Nearly all categories of the results in our proposed method show improvement on target domain adaptation task upon baseline method. Additionally, our method effectively addresses the issue of catastrophic forgetting on source domain by employing adversarial adaptation to learn a domain generalized model. Though our proposed method sometime only see minor improvement over baseline or even fall short in rare cases. We believe that this is because sub-optimal optimization behavior of adversarial training which involves minimaximization on saddle-point.
One way to further improve the performance of our proposed method is to follow with a final stage of SSL fine-tuning. As SSL performance improves with decreasing domain discrepancy\cite{ben2010theory, zhao2019learning, kumar2020understanding}, our proposed method can be used as a pre-processing step for SSL. With a final stage of KD, we would achieve over $2\%$ performance increase among baseline methods in most categories of domain adaptation task.

\vspace{-1em}
\section{Conclusion}
\vspace{-1em}
We have proposed a double-head discriminator algorithm for continual adversarial domain adaptation. With our introduced source-only domain discriminator, the empirical estimation error of the $\gH$-divergence related domain adversarial loss is reduced from source domain side. Extensive experiment has shown that our proposed algorithm has consistently outperformed existing baseline. For future work, we would focus on source-free adversarial domain adaptation algorithm.
\bibliographystyle{unsrt}
\bibliography{ref.bib}
\appendix
\onecolumn
\section{An Overview of Adversarial Domain Adaptation}
\setcounter{theorem}{0}
\label{sec:adv}
Let $T$ and $S$ be the source and target distributions, respectively.
In a general formulation, the upper bound of the target prediction error is given by Ben-David et.al,\cite{ben2010theory}
\begin{theorem}
\label{ben2010theory}
Let $\gF$ be the hypothesis
space. For any classifier $f \in \gF$, $\text{err}_{S}$ denotes the population loss of a classifier $f \in \gF$ under the source distribution $S$, i.e.,
$
    \text{err}_{\mathcal S}(f) \triangleq \mathbb{E}_{(\mathbf{x}_i, \mathbf{y}_i) \sim S}[\ell(f(\mathbf{x}_i), y_i)]
$
And $\text{err}_{T}(f)$ parallel notates for the target domain error. respectively. Then for any classifier $f \in \gH$,
\begin{equation}
\label{upperbound}
    \text{err}_{S}(f) \leq \text{err}_{T}(f)+ d_{\gH\Delta\gH}(\mathcal{T,S}) + \min_{f^{*} \in \mathcal{F}}\{\text{err}_{S}(f^{*})+  \text{err}_{T}(f^{*})\},
\end{equation}
where $d_{\gH\Delta\gH}(\mathcal{T,S})$ is a discrepancy-based distance, known as the $\mathcal{H}$-divergence, and $\min_{f^{*} \in \mathcal{H}}\{\text{err}_{\mathcal P}(f^{\ast})+  \text{err}_{\mathcal Q}(f^{\ast})\}$ is the optimal joint error, i.e., the sum of source and target domain's population loss of $f$ in a {\it{hypothesis}} class $\mathcal{F}$. 
\end{theorem}

For the unsupervised domain adaptation problem, it has been proven that minimizing the upper bound, which is the r.h.s in (\ref{upperbound}), leads to an architecture consisting of a {\it feature extractor} parameterized by $\omega$,  i.e., $f^1_{\omega}$, a {\it label predictor},
parameterized also by $\omega$ i.e., $f^2_{\omega}$ ( $f_{\omega}\triangleq f^2_{\omega}\circ f^1_{\omega}$),
\footnote{The parameters of $f^1$ and $f^2$ are  not the same. In this case,  we abuse the notation to simplify the expression.} 
and a {\it domain classifier}
parameterized by $\psi$,  i.e., $h_{\psi}$, as shown in Fig~\ref{fig:fl_da}~\cite{ganin2015unsupervised,zhao2018adversarial}.
The feature extractor generates the domain-independent feature representations, which are then fed into the domain classifier and label predictor.
The domain classifier then tries to determine whether the extracted features belong to the source or target domain. Meanwhile, the label predictor predicts instance labels based on the extracted features of the labeled source-domain instances.

In Adversarial Domain Adaptation, an additional learning objective of $d_{\gH\Delta\gH}$ is introduced to encourage the extracted features to be both discriminative and invariant to changes between the source and target domains. By extending the $\gH$-divergence to general loss function in \cite{mansour2009domain}, r.h.s in (\ref{upperbound} is equivalent as
\begin{equation}
\min_{\omega}\max_{\psi} \triangleq \E_{(\rvx_i^s, y_i)\sim S}\ell(f_{\omega}(\rvx_i^s),y_i) + \nu\E_{\rvx_i^s\sim S}D^s_{\psi}(\rvx_i^s) +\nu\E_{\rvx_i^t\sim T} D^t_{\psi}(\rvx_i^t)
\end{equation}
where $D^s_{\psi}(\rvx_i^s)\triangleq D^s(h_{\psi}(f_{\omega}^1(\mathbf{x}^s_i))$ and $D^t_{\psi}(\rvx_i^t)\triangleq -D^t(h_{\psi}(f_{\omega}^1(\mathbf{x}^t_i))$

In the majority of domain adversarial problems, $d_{\gH\Delta\gH}$ is reformulated as the difference between the parameterized output of the domain classifier on the source domain and the target domain, given by $\E_{\rvx_i^s\sim S}D^s_{\psi}(\rvx_i^s) + \E_{\rvx_i^t\sim T} D^t_{\psi}(\rvx_i^t)$. This term is commonly referred to as the {\it domain adversarial loss}.

\section{Holistic Regulated One-class Domain Discriminator for continual UDA}
Hu et.al, \cite{hu2020hrn} proposed HRN, a simple but efficient deep one-class learning algorithm. And we adopt HRN as another training oracle for source only domain discriminator with scalar domain discriminators $h_{\psi,s}: \gZ \in\sF \to \sR$ such as DANN\cite{ganin2016domain} and CDAN\cite{long2017conditional}.
\begin{equation}
\label{hrn}
\min_{\psi_s}\E_{\rvx_i^s \sim S}[-\log(\sigmoid(h_{\psi,s}((\rvz(\rvx_i^s)))))] + \lambda \| \nabla_{\rvz} h_{\psi,s}(\rvz(\rvx_i^s))||_2^n
\end{equation}
where $\rvz$ is the domain features that are fed as the input to domain discriminator. In general, the domain features that could be used for $\rvz$ include but not limited to the following cases:
\begin{itemize}
\item Domain-Adversarial Neural Networks (DANN)\cite{ganin2016domain}, $\rvz$ is designed
simply to be the domain invariant feature $f_{\omega}^1(\rvx_i^s)$
\begin{equation}
\rvz \triangleq f_{\omega}^1(\rvx_i^s)
\end{equation}
\item Conditional Domain Adaptation Network (CDAN)\cite{long2017conditional}, $\rvz$ is from the cross-product space of $f_{\omega}^1(\rvx_i^s)$ and $f_{\omega}(\rvx_i^s)$
\begin{equation}
\rvz \triangleq f_{\omega}^1(\rvx_i^s) \otimes f_{\omega}^1(\rvx_i^s)
\end{equation}
\end{itemize}
Apart from the commonly adopted NLL for classification, HRN adds an additional regularization term on the $n$'s power of scalar domain discriminator $h_{\psi,s}(\cdot)$'s gradient norm.And $n$ is the exponential term which is used with $\lambda$ to control the strength of regularization.The full description of our double-head domain discriminator algorithm for scalar domain discriminator  is shown in Algorithm \ref{alg:doubleheadalgscalardc}.
\begin{algorithm}
 \caption{Double Head Discriminator Algorithm}\label{alg:doubleheadalg}
\begin{algorithmic}[1]
    \State Initialization: 
    Task Model $f_{\omega} \triangleq f^2_{\omega}\circ f^1_{\omega}$\\
    Source Only multi-class Domain Classifier $h_{\psi,{s}}$
    \phase{Source Only training phase}
    \Procedure {Task Model Training Phase}{}
    \For{$t \in \{1, \ldots, t_1\}$}
      \For{$\{(\mathbf{x}_1, y_1), \ldots, (\mathbf{x}_K, y_K)\}\stackrel{K}{\sim} S_0$}
      \State $L=\frac{1}{K}\sum_{i=i}^{K}\ell(f_{\omega}(\mathbf{x}_i), y_i)$
      \State $\omega \to SGD(L, \omega)$\Comment{{Train Task Model on Source Domain}}
      \EndFor
    \EndFor
    \EndProcedure
    \Procedure {Source Only Domain Classifier Training Phase}{}
    \For{$t \in \{1, \ldots, t_2\}$}
      \For{$(\mathbf{x}_1, \ldots, \mathbf{x}_K) \stackrel{K}{\sim} S_0$}
      \State $d'(\mathbf{x})\to \arg\max_cf_{\omega}(\mathbf{x}_i, c) \hspace{2em} \forall \mathbf{x} \in \{\mathbf{x}_1\ldots \mathbf{x}_K\}$ \Comment{{Get Pseudo Domain Label from Task Model }}
      \State $D=\frac{1}{K}\sum_{i=i}^{K}-\log(\softmax(h_{\psi, s}(f^1_{\omega}(\mathbf{x}_i)), d'(\mathbf{x}_i)))$
      \State $\psi_s \to SGD(D, \psi_s)$\Comment{{Train on Source Only Domain Classifier}}
      \EndFor
    \EndFor
    \EndProcedure
    \phase{Sample on Source Domain Replay Memory}
    \Procedure{Memory Sample Phase}{}
    \State $\mathcal{M} \to \{\}$
    \For{$c \in \{1, \ldots, C\}$}
    \State Sample $\{(\mathbf{x}_1, c), \ldots, (\mathbf{x}_N, c)\}\stackrel{N}{\sim} S_0$
    \State $\mathcal{M}.append(\{(\mathbf{x}_1, c), \ldots, (\mathbf{x}_N, c)\})$\Comment{{Store N data per class on source domain on Replay Memory}}
    \EndFor
    \EndProcedure
    \phase{Unlabeled Target Adaptation Phase with Memory Reply}
    \State Initialization:Target Adaptation Phase multi-class Domain Classifier $h_{\psi,{t}}$
    \Procedure{Target Phase}{}
      \For{$t \in \{1, \ldots, t_3\}$}
        \For{$\{(\mathbf{x}^s_1, y^s_1), \ldots, (\mathbf{x}^s_K, y^s_K)\}\stackrel{K}{\sim} \mathcal{M}$, $(\mathbf{x}^t_1\ldots, \mathbf{x}^t_K)\stackrel{K}{\sim} T_1$} 
        \State $L=\frac{1}{K}\sum_{i=1}^{K}\ell(f_{\omega}(\mathbf{x}_i^s), y_i^s)$
        \State $\mathbf{d}'(\mathbf{x})\to \arg\max_cf_{\omega}(\mathbf{x}, c) \hspace{1em}\forall \mathbf{x} \in \{\mathbf{x}^s_1\ldots \mathbf{x}^s_K,\mathbf{x}^t_1\ldots \mathbf{x}^t_K\}$
        \State $D_{\psi, t} = \frac{1}{K}\sum_{i=1}^{K}-\log(\softmax(h_{\psi, t}(f^1_{\omega}(\mathbf{x}^s_i)), \mathbf{d}'(\mathbf{x}^s_i))) - \log(1- \softmax(h_{\psi, t}(f^1_{\omega}(\mathbf{x}^t_i)), \mathbf{d}'(\mathbf{x}^t_i)))$
        \State $D_{\psi} = \frac{1}{K}\sum_{i=1}^{K}-\log(\softmax(h_{\psi, s}(f^1_{\omega}(\mathbf{x}^s_i))+h_{\psi, t}(f^1_{\omega}(\mathbf{x}^s_i)), \mathbf{d}'(\mathbf{x}^s_i))) - \log(1- \softmax(h_{\psi, s}(f^1_{\omega}(\mathbf{x}^t_i))+h_{\psi, t}(f^1_{\omega}(\mathbf{x}^t_i)), \mathbf{d}'(\mathbf{x}^t_i)))$
        \State $\omega \to SGD(L-\beta D_\psi, \omega)$
        \State $\psi_t \to SGD(D_{\psi,t}, \psi_t)$
      \EndFor
    \EndFor
    \EndProcedure
  \end{algorithmic}
\end{algorithm}

\begin{algorithm}
 \caption{Double Head Discriminator Algorithm For Scalar Domain Discriminator}\label{alg:doubleheadalgscalardc}
  \begin{algorithmic}
  \State Initialization: 
    Task Model $f_{\omega} \triangleq f^2_{\omega}\circ f^1_{\omega}$\\
    Source Only scalar Domain Classifier $h_{\psi,{s}}$
    \phase{Source Only training phase}
    \Procedure {Task Model Training Phase}{}
    \For{$t \in \{1, \ldots, t_1\}$}
      \For{$\{(\mathbf{x}_1, y_1), \ldots, (\mathbf{x}_K, y_K)\}\stackrel{K}{\sim} S_0$}
      \State $L=\frac{1}{K}\sum_{i=i}^{K}\ell(f_{\omega}(\mathbf{x}_i), y_i)$
      \State $\omega \to SGD(L, \omega)$\Comment{{Train Task Model on Source Domain}}
      \EndFor
    \EndFor
    \EndProcedure
    \Procedure {Source Only Domain Classifier Training Phase}{}
    \For{$t \in \{1, \ldots, t_2\}$}
      \For{$(\mathbf{x}_1, \ldots, \mathbf{x}_K) \stackrel{K}{\sim} S_0$}
      \State $D=\frac{1}{K}\sum_{i=i}^{K}-\log(\sigmoid(h_{\psi, s}(f^1_{\omega}(\mathbf{x}_i)),)+  \lambda \| \nabla_{\rvz} h_{\psi,s}(\rvz)|_{\rvz=f_{\omega}^1(\rvx_i)}\|_2^n$
      \State $\psi_s \to SGD(D, \psi_s)$\Comment{{Train on Source Only Domain Classifier}}
      \EndFor
    \EndFor
    \EndProcedure
    \phase{Sample on Source Domain Replay Memory}
    \Procedure{Memory Sample Phase}{}
    \State $\mathcal{M} \to \{\}$
    \For{$c \in \{1, \ldots, C\}$}
    \State Sample $\{(\mathbf{x}_1, c), \ldots, (\mathbf{x}_N, c)\}\stackrel{N}{\sim} S_0$
    \State $\mathcal{M}.append(\{(\mathbf{x}_1, c), \ldots, (\mathbf{x}_N, c)\})$\Comment{{Store N data per class on source domain on Replay Memory}}
    \EndFor
    \EndProcedure
    \phase{Unlabeled Target Adaptation Phase with Memory Reply}
    \State Initialization:Target Adaptation Phase scalar Domain Classifier $h_{\psi,{t}}$
    \Procedure{Target Adaptation Phase}{}
      \For{$t \in \{1, \ldots, t_3\}$}
        \For{$\{(\mathbf{x}^s_1, y^s_1), \ldots, (\mathbf{x}^s_K, y^s_K)\}\stackrel{K}{\sim} \mathcal{M}$, $(\mathbf{x}^t_1\ldots, \mathbf{x}^t_K)\stackrel{K}{\sim} T_1$} 
        \State $L=\frac{1}{K}\sum_{i=i}^{K}\ell(f_{\omega}(\mathbf{x}^s_i), y_i)$
        \State $D_{\psi, t} = \frac{1}{K}\sum_{i=i}^{K}-\log(\sigmoid(h_{\psi, t}(f^1_{\omega}(\mathbf{x}^s_i)))) - \log(\sigmoid(-h_{\psi, t}(f^1_{\omega}(\mathbf{x}^t_i))))$
        \State $D_{\psi} = \frac{1}{K}\sum_{i=i}^{K}-\log(\sigmoid(h_{\psi, s}(f^1_{\omega}(\mathbf{x}^s_i))+h_{\psi, t}(f^1_{\omega}(\mathbf{x}^s_i)))) - \log(\sigmoid(-h_{\psi, s}(f^1_{\omega}(\mathbf{x}^t_i))-h_{\psi, t}(f^1_{\omega}(\mathbf{x}^t_i))))$
        \State $\omega \to SGD(L-\beta D_\psi, \omega)$
        \State $\psi_t \to SGD(D_{\psi,t}, \psi_t)$
      \EndFor
    \EndFor
    \EndProcedure
  \end{algorithmic}
\end{algorithm}

\section{Experiment Setup}
\label{expsetup}
\textbf{MNISTM} \cite{ganin2016domain}
is a dataset that demonstrates domain adaptation by combining MNIST with randomly colored image patches from the BSD500 dataset.

\noindent\textbf{USPS} \cite{hull1994database} 
is a digit dataset automatically scanned from envelopes by the U.S. Postal Service containing pixel grayscale samples. The images are centered, normalized. And a broad range of font styles are shown in the dataset.


\noindent\textbf{SVHN} has RGB images of printed digits clipped from photographs of house number plates. The trimmed photos are centered on the digit of interest while surrounding digits and other distractions are retained. Photos of house numbers in various countries was used to create the SVHN dataset.

\noindent\textbf{Office-31} \cite{saenko2010adapting} is a typical domain adaptation dataset made up of three distinct domains with 31 categories in each domain. 
There are 4,652 images in total from 31 classes.

\noindent\textbf{Office-home} \cite{venkateswara2017deep} is a typical domain adaptation dataset made up of four distinct domains with 65 categories in each domain. There are total 15,500 images in total from 65 classes

\noindent\textbf{Implementation Details} On MNISTM, USPS and SVHN, we use a three-layer convolutional network as the invariant feature extractor, and the network models are trained from random initialization on server. On Office-31 and Office-home, we use the pre-trained ResNet50\cite{he2016deep} on ImageNet\cite{russakovsky2015imagenet} as the feature extractor. Both the task classifier and the domain classifier are  two-layer fully-connected neural networks.  The domain classifier's parameter are trained from random initialization in all settings. In Office-31 and Office-home datasets, we set the memory buffer size as 10 samples per class. We uniformly use supervised training in source domain data for 15 epochs in $S_0$ phase. Following supervised supervised training, we train freeze our task model and train source only domain classifier for 5 epochs in all remaining experiment except for ablation study. Learning rate of task model and domain discriminator is fixed with 0.001 on Adam optimizer. The source only domain discriminator is trained with Adam optimizer of learning rate 0.0001 for 5 epochs in all remaining experiment except for ablation study.

\section{Additional Experiment Result on Holistic Regulated One-class Domain Discriminator}
\begin{table}
\centering
\caption{Holistic Regulated One-class Domain Discriminator.}
\begin{tabular}{lrrrrrrr|c|}
\toprule
\textbf{Discriminator Used}& DANN($T_1$Only) &  CDAN($T_1$Only) & HRN-DANN & HRN-CDAN\\
\midrule
MNIST $\to$ MNISTM &$58.8 $ &  $59.2$  & $78.1$ & $80.3$ \\
MNIST $\to$ USPS &$60.6$  & $62.3$ & $69.1$  & $73.4$\\
MNIST $\to$ SVHN &$32.1$ & $35.7$ &  $37.5$  & $40.8$\\
\bottomrule
\end{tabular}
\label{tab:result_hrn}
\end{table}
Table 4: Experiment on testing with holistic-regulated training on source-only domain discriminator. The source-only domain discriminator $h_{\psi,s}$ of scalar output is trained on holistic regulated one-class loss in Equation (\ref{hrn}). The rest of domain adversarial training is the same as in Algorithm (\ref{alg:doubleheadalgscalardc}) using the ensembles of two discriminators digits as domain invariant signals for feature extractor, $f^1_{\omega}$. $n=6$ and $\lambda=0.1$ is used as holistic regulated loss on $h_{\psi,s}$ for stable performance. In general, including a holistic regulated source-only domain discriminator has performance improvement over using single domain discriminator in $T_1$ only. However the HRN method of training a scalar domain discriminator is inferior than MDD included multi-class domain discriminator for continual UDA. 

\setcounter{figure}{0}
\label{app:flowchart}
\begin{figure}
\centering
\includegraphics[width=0.8\textwidth]{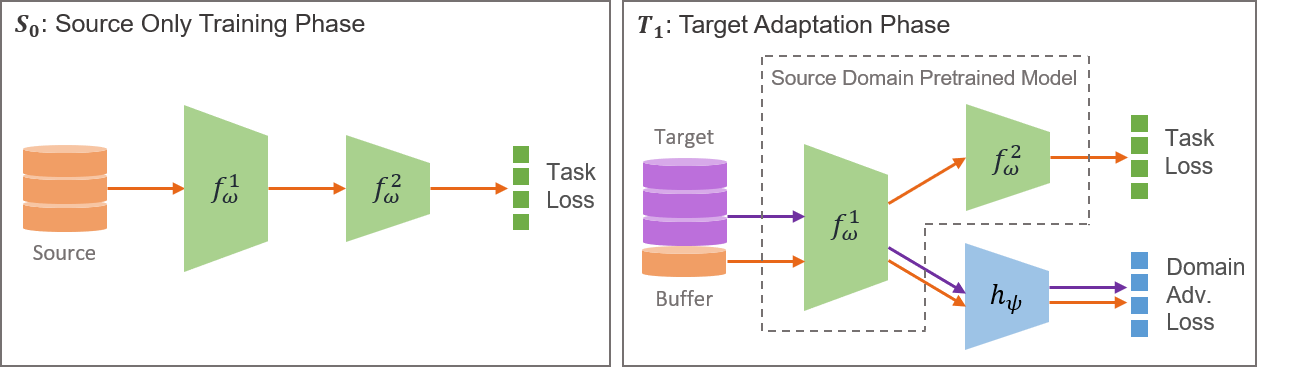}
\caption{{\small A continual adversarial domain adaptation model. Only the source risk of the client's local source data is accessible in source only training phase. A small set of buffered source domain data and target domain data is adversarial trained in target adaptation phase.}}
\label{fig:fl_da}
\end{figure}
 
\section{Proof of Theorem 2}
\begin{lemma}
(Lemma C.1, \cite{zhang2019bridging}) For any distribution $D$ and any $f$, we have
\begin{equation}
disp_{D}^{(\rho)}(f',f)= err_{D}^{(\rho)}(f')+err_{D}^{(\rho)}(f)
\end{equation}
\end{lemma}
\begin{theorem}
For a hypothesis class $\gF$ and a fixed $f_0\in\gF$ where for every $f \in \gF$, $f-f_0$ is also in $\gF$, then we have the following property holds
\begin{equation}
err_T(f) \leq err_S^{(\rho)}(f) + d^{(\rho)}_{f,f_0, \gF}(S,T) + \lambda
\end{equation}
where $err_S^{(\rho)}(f)$, $d^{(\rho)}_{f,f_0, \gF}(S,T)$ and $\lambda$ is defined as
\begin{equation}
\begin{split}
&err_S^{(\rho)}(f)= \E_{(x_i,y_i)\sim S}\Phi_{\rho}\circ\rho_{f}(x_i, y_i)\\
&d^{(\rho)}_{f,f_0, \gF}(S,T)=\sup_{f'\in\gF}\{
\E_{x_i\sim T}\Phi_{\rho}\circ\rho_{f'+f_0}(x_i, h_f(x_i))-\E_{x_i\sim S}\Phi_{\rho}\circ\rho_{f'+f_0}(x_i, h_f(x_i))\}\\
&\lambda=\min_{f^{\star}\in\gF}err_S^{(\rho)}(f^{\star})+err_T^{(\rho)}(f^{\star}),
\end{split}
\end{equation}
\end{theorem}
\begin{proof}
We first define $f^{\star}$ be the ideal joint hypothesis which minimizes the combined margin loss,
\begin{equation}
f^{\star} \triangleq \argmin_{f\in\gF}\{err_S^{(\rho)}(f) + err_T^{(\rho)}(f)\}
\end{equation}
\begin{equation}
\begin{split}
err_T(f) \leq & \E_T \mathbbm{1}[h_f \neq h_{f^{\star}}] +  \E_T \mathbbm{1}[h_{f^{\star}} \neq y] \\
\leq & err_S^{(\rho)}(f) - err_S^{(\rho)}(f) + disp_T^{(\rho)}(f^{\star}, f) + err_T^{(\rho)}(f^{\star})
\end{split}
\end{equation}
From the triangular inequality of margin discrepancy(Lemma C.1, \cite{zhang2019bridging}), we have
\begin{equation}
\label{errrhobound}
\begin{split}
err_T(f) \leq & err_S^{(\rho)}(f) - err_S^{(\rho)}(f) + disp_T^{(\rho)}(f^{\star}, f) + err_T^{(\rho)}(f^{\star}) \\
\leq & err_S^{(\rho)}(f) + err_S^{(\rho)}(f^{\star}) - disp_S^{(\rho)}(f^{\star}, f) + disp_T^{(\rho)}(f^{\star},f) + err_T^{(\rho)}(f^{\star})
\end{split}
\end{equation}
Let we define $f_1 \triangleq f^{\star} - f_0$. From the properties of hypothesis class $\gF$, we have $f_1 \in \gF$. By substituting the definition of $f_1$ into $disp_S^{(\rho)}(f^{\star}, f)$ and $disp_T^{(\rho)}(f^{\star}, f)$, we have
\begin{equation}
\label{defdrho}
\begin{split}
&disp_T^{(\rho)}(f^{\star}, f) - disp_S^{(\rho)}(f^{\star},f)= \E_{x_i\sim T}\Phi_{\rho}\circ\rho_{f^{\star}}(x_i, h_f(x_i))-\E_{x_i\sim S}\Phi_{\rho}\circ\rho_{f^{\star}}(x_i, h_f(x_i)) \\
&= \E_{x_i\sim T}\Phi_{\rho}\circ\rho_{f_1+f_0}(x_i, h_f(x_i))-\E_{x_i\sim S}\Phi_{\rho}\circ\rho_{f_1+f_0}(x_i, h_f(x_i)) \\
&\leq \sup_{f'\in\gF}\{
\E_{x_i\sim T}\Phi_{\rho}\circ\rho_{f'+f_0}(x_i, h_f(x_i))-\E_{x_i\sim S}\Phi_{\rho}\circ\rho_{f'+f_0}(x_i, h_f(x_i))\} \\
& = d^{(\rho)}_{f,f_0, \gF}(S,T)
\end{split}
\end{equation}
By substituting Eq (\ref{defdrho}) into Eq (\ref{errrhobound}), we finally reach
\begin{equation}
err_T(f) \leq err_S^{(\rho)}(f) + d^{(\rho)}_{f,f_0, \gF}(S,T) + \lambda
\end{equation}
\end{proof}
\section{Proof of Theorem 3}
\begin{lemma}
\label{recomplexlemma}
(Theorem 3.3, \cite{mohri2018foundations}) Let $\gG$ be a family of functions mapping $\gX \in \sD \to \sR$. Then for any $\delta>0$, with probability at least $1-\delta$ over the draw of i.i.d samples from sample $S$ of size $m$, each of the following holds for all $g\in \gG$
\begin{equation}
\E[g(z)]  \leq \frac{1}{m}\sum_{i=1}^{m}g(z_i)+2\Re_m(\gG)+\sqrt{\frac{\log\frac{1}{\delta}}{2m}}
\end{equation}
\end{lemma}
\begin{lemma}
\label{talagrandlemma}
(Talagrand’s lemma, \cite{mohri2018foundations}) Let $\Psi_i: \R \to \R$ be an $l$-Lipschitz. Then for any hypothesis set $\gG$ of real valued functions, and for any sample $D$ of size $n$, the following inequality holds:
\begin{equation}
\E_{\delta}\sup_{g\in\gG}\frac{1}{n}\sum_{i=1}^{n}\delta_i(\Psi_i\circ g)(\rvx_i) \leq l\hat{\Re}_{n, \hat{D}}(\gG) 
\end{equation}
\end{lemma}
\begin{theorem}
Let $f_0 \in \gF$ be a fixed source function that maps from $\gX \times \gY \to \sR$ that is trained on source domain only which satisfies $\rho_{f_0}(\rvx^s, h_f)\geq \epsilon_s$ for source domain data $\rvx^s \in S$ and $\rho_{f_0}(\rvx^t, h_f) \leq \epsilon_t$ for target domain data as outliers $\rvx^t \in T$. $\rvx_i^s$ is an i.i.d sample of size m drawn from the source distribution S and $\rvx_t^s$ is an i.i.d sample of size n drawn from the target distribution T. Given the same settings as Definition \ref{rademacher}. For any $\delta>0$, with the probability at least $1-2\delta$, we have the following generalization error bound for domain discrepancy loss function 
\begin{equation}
\begin{split}
&\E_{\rvx^s\in S}[\log(\frac{e^{\rho_{f'}(\rvx^s, h_f)+ \rho_{f_0}(\rvx^s, h_f)}}{1+e^{\rho_{f'}(\rvx^s, h_f)+\rho_{f'}(\rvx^s, h_f)}})] + \E_{\rvx^t\in T}[\log(\frac{1}{1+e^{\rho_{f'}(\rvx^t, h_f)+\rho_{f_0}(\rvx^t, h_f)}})]\\
&\leq  \frac{1}{m}\sum_{i=1}^{m}\log(\frac{e^{\rho_{f'}(\rvx^s_i, h_f)+ \rho_{f_0}(\rvx^s_i, h_f)}}{1+e^{\rho_{f'}(\rvx^s_i, h_f)+\rho_{f'}(\rvx^s_i, h_f)}})+ \frac{1}{n}\sum_{i=1}^{n}\log(\frac{1}{1+e^{\rho_{f'}(\rvx^t_i, h_f)+\rho_{f'}(\rvx^t_i, h_f)}})\\
&+ \max\{\frac{2}{(e^{\epsilon_s}-1)\lambda_s^++1}, \frac{2}{(e^{\epsilon_s}-1)\lambda_s^-+1}\}{\Re}_{m, \gD_s}(\gG_s) \\
&+ \max\{\frac{2e^{\epsilon_t}}{(1-\lambda_t^+)e^{\epsilon_t}+\lambda_t^+},\frac{2e^{\epsilon_t}}{(1-\lambda_t^-)e^{\epsilon_t}+\lambda_t^-}\}{\Re}_{n, \gD_t}(\gG_t) + \sqrt{\frac{\log\frac{1}{\delta}}{2m}} + \sqrt{\frac{\log\frac{1}{\delta}}{2n}}
\end{split}
\end{equation}
where $\lambda_s$ and $\lambda_t$ is defined as
\begin{equation}
\begin{split}
\lambda_s^- &= \min\{\frac{e^{\rho_{f'}(\rvx^s, h_f)}}{1+e^{\rho_{f'}(\rvx^s, h_f)}}\},  \lambda_s^+ = \max\{\frac{e^{\rho_{f'}(\rvx^s, h_f)}}{1+e^{\rho_{f'}(\rvx^s, h_f)}}\}, \forall \rvx^s \in S  \\ 
\lambda_t^- &= \min\{\frac{1}{1+e^{\rho_{f'}(\rvx^t, h_f)}}\}, \lambda_t^+ = \max\{\frac{1}{1+e^{\rho_{f'}(\rvx^t, h_f)}}\}, \forall \rvx^t \in T 
\end{split}
\end{equation}
\end{theorem}
\begin{proof}
We first define $z^s_i$ as
\begin{equation}
\label{zsdef}
 z^s_i \triangleq \log(\frac{e^{\rho_{f'}(\rvx_i, h_f)}}{1+e^{\rho_{f'}(\rvx_i, h_f)}})
\end{equation}
From the above Equation, we have
\begin{equation}
e^{\rho_{f'}(\rvx_i, h_f)}= \frac{e^{z^s_i}}{1-e^{z^s_i}} 
\end{equation}
Then by substituting the above equation into $\log(\frac{e^{\rho_{f'}(\rvx, h_f)+\rho_{f_0}(\rvx_i, h_f)}}{1+e^{\rho_{f'}(\rvx, h_f)+ \rho_{f_0}(\rvx_i, h_f)}})$, we have
\begin{equation}
\log(\frac{e^{\rho_{f'}(\rvx, h_f)+\rho_{f_0}(\rvx_i, h_f)}}{1+e^{\rho_{f'}(\rvx, h_f)+ \rho_{f_0}(\rvx_i, h_f)}}) = \log(\frac{e^{z^s_i+\rho_{f_0}(\rvx_i, h_f)}}{e^{z^s_i+\rho_{f_0}(\rvx_i, h_f)} +1-e^{z_i^s}})
\end{equation}
Define the following transformation function
\begin{equation}
\Gamma_i(z_i^s) = \log(\frac{e^{z^s_i+\rho_{f_0}(\rvx_i, h_f)}}{e^{z^s_i+\rho_{f_0}(\rvx_i, h_f)} +1-e^{z_i^s}})
\end{equation}
By Lemma \ref{recomplexlemma}, with probability at least $1-\delta$, for any $g_s \in \gG_s$.
\begin{equation}
\label{sourcedomainlossbound}
\E_{\rvx^s\in S}[\log(\frac{e^{\rho_{f'}(\rvx^s, h_f)+ \rho_{f_0}(\rvx^s, h_f)}}{1+e^{\rho_{f'}(\rvx^s, h_f)+\rho_{f'}(\rvx^s, h_f)}})] - \frac{1}{m}\sum_{i=1}^{m}\log(\frac{e^{\rho_{f'}(\rvx^s_i, h_f)+ \rho_{f_0}(\rvx^s_i, h_f)}}{1+e^{\rho_{f'}(\rvx^s_i, h_f)+\rho_{f'}(\rvx^s_i, h_f)}}) \leq 2\Re_{m, \gD_s}(\Gamma_i \circ \gG_s ) + \sqrt{\frac{\log\frac{1}{\delta}}{2m}}
\end{equation}
Next we take gradient on $\Gamma_i$ 
\begin{equation}
\Gamma_i'(z_i^s) = \frac{1}{e^{z_i^s+\rho_{f_0}(\rvx_i, h_f)}+1-e^{z_i^s}}
\end{equation}
From the definition of $z_i^s, \lambda_s, \epsilon_s$, we have
\begin{equation}
0 \leq e^{z_i^s} \leq 1, \hspace{1em} \rho_{f_0}(\rvx_i, h_f) \geq \epsilon_s
\end{equation}
Then we can bound $\Gamma_i'$ by
\begin{equation}
0 \leq \Gamma_i'(z_i^s) \leq \frac{1}{(e^{\epsilon_s}-1)e^{z_i^s}+1}
\end{equation}
As $e^{z_i^s}$ takes value between $[\lambda_s^-, \lambda_s^+]$, using the properties of linear functions, we have
\begin{equation}
0 \leq \Gamma_i'(z_i^s) \leq l_{\Gamma} = \max\{\frac{1}{(e^{\epsilon_s}-1)\lambda_s^-+1}, \frac{1}{(e^{\epsilon_s}-1)\lambda_s^++1}\}
\end{equation}
Therefore $\Gamma_i$ is $l_{\Gamma}$-Lipschitz. By applying the Lemma \ref{talagrandlemma} into inequality (\ref{sourcedomainlossbound}), we have the following inequality holds with probability at least $1-\delta$
\begin{equation}
\label{sourcetdomainlossbound1}
\begin{split}
&\E_{\rvx^s\in S}[\log(\frac{e^{\rho_{f'}(\rvx^s, h_f)+ \rho_{f_0}(\rvx^s, h_f)}}{1+e^{\rho_{f'}(\rvx^s, h_f)+\rho_{f'}(\rvx^s, h_f)}})] - \frac{1}{m}\sum_{i=1}^{m}\log(\frac{e^{\rho_{f'}(\rvx^s_i, h_f)+ \rho_{f_0}(\rvx^s_i, h_f)}}{1+e^{\rho_{f'}(\rvx^s_i, h_f)+\rho_{f'}(\rvx^s_i, h_f)}}) \leq 2\Re_{m, \gD_s}(\Gamma_i \circ \gG_s ) + \sqrt{\frac{\log\frac{1}{\delta}}{2m}}\\
& \leq \max\{\frac{2}{(e^{\epsilon_s}-1)\lambda_s^++1}, \frac{2}{(e^{\epsilon_s}-1)\lambda_s^-+1}\}{\Re}_{m, \gD_s}(\gG_s) + \sqrt{\frac{\log\frac{1}{\delta}}{2m}}
\end{split}
\end{equation}
Similarly we define $z_i^t$
\begin{equation}
\label{ztdef}
 z^t_i \triangleq \log(\frac{1}{1+e^{\rho_{f'}(\rvx_i, h_f)}})
\end{equation}
From the above Equation, we have
\begin{equation}
e^{\rho_{f'}(\rvx_i, h_f)}= e^{-z^t_i}-1
\end{equation}
Then by substituting the above equation into $\log(\frac{1}{1+e^{\rho_{f'}(\rvx, h_f)+ \rho_{f_0}(\rvx_i, h_f)}})$, we have
\begin{equation}
\log(\frac{1}{1+e^{\rho_{f'}(\rvx, h_f)+ \rho_{f_0}(\rvx_i, h_f)}}) = \log(\frac{1}{e^{z^t_i+\rho_{f_0}(\rvx_i, h_f)} +1-e^{z_i^t}})
\end{equation}
Similarly we define the following transformation function
\begin{equation}
\Psi_i(z_i^t) = \log(\frac{1}{e^{z^t_i+\rho_{f_0}(\rvx_i, h_f)} +1-e^{z_i^t}})
\end{equation}
By Lemma \ref{recomplexlemma}, with probability at least $1-\delta$, for any $g_t\in \gG_t$.
\begin{equation}
\label{targetdomainlossbound}
\E_{\rvx^t\in T}[\log(\frac{1}{1+e^{\rho_{f'}(\rvx^t, h_f)+\rho_{f_0}(\rvx^t, h_f)}})] - \frac{1}{n}\sum_{i=1}^{n}\log(\frac{1}{1+e^{\rho_{f'}(\rvx^t_i, h_f)+\rho_{f'}(\rvx^t_i, h_f)}}) \leq 2\Re_{n, \gD_t}(\Psi_i \circ \gG_t ) + \sqrt{\frac{\log\frac{1}{\delta}}{2n}}
\end{equation}
Next we take gradient on $\Psi_i$
\begin{equation}
\Psi_i'(z_i^t) = \frac{e^{\rho_{f_0}(\rvx_i, h_f)}}{e^{\rho_{f_0}(\rvx_i, h_f)}-e^{\rho_{f_0}(\rvx_i, h_f)+z_i^t}+e^{z_i^t}}
\end{equation}
From the definition of $z_i^t, \lambda_t, \epsilon_t$, we have
\begin{equation}
0 \leq e^{z_i^t} \leq 1, \hspace{1em} \rho_{f_0}(\rvx_i, h_f) \leq \epsilon_t
\end{equation}
Then we can bound $\Psi_i'$ by
\begin{equation}
0 \leq \Psi_i'(z_i^t) \leq \frac{e^{\epsilon_t}}{(1-e^{z_i^t})e^{\epsilon_t}+e^{z_i^t}}
\end{equation}
As ${e^{\epsilon_t}}$ takes value between $[\lambda_t^-, \lambda_t^+]$, using properties of linear functions, we have
\begin{equation}
0 \leq \Psi_i'(z_i^t) \leq l_{\Psi} = \max\{\frac{e^{\epsilon_t}}{(1-\lambda^-_t)e^{\epsilon_t}+\lambda^-_t}, \frac{e^{\epsilon_t}}{(1-\lambda^+_t)e^{\epsilon_t}+\lambda^+_t}\}
\end{equation}
Therefore $\Psi_i$ is $l_{\Psi}$-Lipschitz.  By applying the Lemma \ref{talagrandlemma} into Inequality (\ref{targetdomainlossbound}), we have the following inequality holds with probability at least $1-\delta$
\begin{equation}
\label{targetdomainlossbound1}
\begin{split}
&\E_{\rvx^t\in T}[\log(\frac{1}{1+e^{\rho_{f'}(\rvx^t, h_f)+\rho_{f_0}(\rvx^t, h_f)}})] - \frac{1}{n}\sum_{i=1}^{n}\log(\frac{1}{1+e^{\rho_{f'}(\rvx^t_i, h_f)+\rho_{f'}(\rvx^t_i, h_f)}}) \leq 2\Re_{n, \gD_t}(\Psi_i \circ \gG_t ) + \sqrt{\frac{\log\frac{1}{\delta}}{2n}}\\
 &\leq \max\{\frac{2e^{\epsilon_t}}{(1-\lambda^-_t)e^{\epsilon_t}+\lambda^-_t}, \frac{2e^{\epsilon_t}}{(1-\lambda^+_t)e^{\epsilon_t}+\lambda^+_t}\}{\Re}_{n, \gD_t}(\gG_t) + \sqrt{\frac{\log\frac{1}{\delta}}{2n}}
\end{split}
\end{equation}
By summing up Equation (\ref{targetdomainlossbound1}) with (\ref{sourcetdomainlossbound1}), we have the following inequality holds with probability at least $1-2\delta$
\begin{equation}
\begin{split}
&\E_{\rvx^s\in S}[\log(\frac{e^{\rho_{f'}(\rvx^s, h_f)+ \rho_{f_0}(\rvx^s, h_f)}}{1+e^{\rho_{f'}(\rvx^s, h_f)+\rho_{f'}(\rvx^s, h_f)}})] + \E_{\rvx^t\in T}[\log(\frac{1}{1+e^{\rho_{f'}(\rvx^t, h_f)+\rho_{f_0}(\rvx^t, h_f)}})]\\
&\leq  \frac{1}{m}\sum_{i=1}^{m}\log(\frac{e^{\rho_{f'}(\rvx^s_i, h_f)+ \rho_{f_0}(\rvx^s_i, h_f)}}{1+e^{\rho_{f'}(\rvx^s_i, h_f)+\rho_{f'}(\rvx^s_i, h_f)}})+ \frac{1}{n}\sum_{i=1}^{n}\log(\frac{1}{1+e^{\rho_{f'}(\rvx^t_i, h_f)+\rho_{f'}(\rvx^t_i, h_f)}})\\
&+ \max\{\frac{2}{(e^{\epsilon_s}-1)\lambda_s^++1}, \frac{2}{(e^{\epsilon_s}-1)\lambda_s^-+1}\}{\Re}_{m, \gD_s}(\gG_s) \\
&+ \max\{\frac{2e^{\epsilon_t}}{(1-\lambda_t^+)e^{\epsilon_t}+\lambda_t^+},\frac{2e^{\epsilon_t}}{(1-\lambda_t^-)e^{\epsilon_t}+\lambda_t^-}\}{\Re}_{n, \gD_t}(\gG_t) 
+ \sqrt{\frac{\log\frac{1}{\delta}}{2m}} + \sqrt{\frac{\log\frac{1}{\delta}}{2n}}
\end{split}
\end{equation}
which completes the proof
\end{proof}
\section{Proof of Proposition 1}
\setcounter{proposition}{0}

\begin{proposition}
Consider the following optimization problem we have defined
\begin{align}
\label{maxcond}
&\max_{f'}\E_{\hat{S}}\log(\sigmoid_{h_f}\circ f')+\E_{\hat{T}}\log(1-\sigmoid_{h_f}\circ f')\\
\label{mincond}
&\min_{\hat{S},\hat{T}}\E_{\hat{S}}\log(\frac{1}{2}\sigmoid_{h_f}\circ f'+\frac{1}{2}\sigmoid_{h_f}\circ f_0)+\E_{\hat{T}}\log(1-\frac{1}{2}\sigmoid_{h_f}\circ f'-\frac{1}{2}\sigmoid_{h_f}\circ f_0)
\end{align}
Assume that there is no restriction on the choice of $f'$. By fixing $f_0$, we have the following two results.
\begin{enumerate}
\item The optimal value of $\sigmoid_{h_f}\circ f'$ on data x is 
\begin{equation}
\frac{\hat{p}_s(\rvx)}{\hat{p}_s(\rvx)+\hat{q}_t(\rvx)}
\end{equation}
where $\hat{p}_s(\rvx)$ and $\hat{q}_t(\rvx)$ are density functions of source and target domain empirical distributions
\item The minimization problem w.r.t $S$ and $T$ is equivalent to minimization on the sum of two terms $L_1$ and $L_2$, where
\begin{equation}
\begin{split}
    L_1 &= 4KL(\frac{3}{4}\hat{T}+\frac{1}{4}\hat{S}||\frac{1}{2}\hat{T}+\frac{1}{2}\hat{S}) + 4KL(\frac{1}{2}\hat{T}+\frac{1}{2}\hat{S}||\frac{3}{4}\hat{T}+\frac{1}{4}\hat{S})\\
     &+ 4KL(\frac{3}{4}\hat{S}+\frac{1}{4}\hat{T}||\frac{1}{2}\hat{T}+\frac{1}{2}\hat{S}) + 4KL(\frac{1}{2}\hat{T}+\frac{1}{2}\hat{S}||\frac{3}{4}\hat{S}+\frac{1}{4}\hat{T})
\end{split}
\end{equation}
is a symmetric distribution divergence between $\hat{S}$ and $\hat{T}$ and has global minimum of $\hat{S}=\hat{T}$
\begin{equation}
L_2 =\int_{\rvx}(1-2\sigmoid_{h_f}\circ f_0(\rvx))(\hat{q}_t(\rvx)-\hat{p}_s(\rvx))\frac{1}{4-(\hat{p}_s(\rvx)-\hat{q}_t(\rvx))^2/(\hat{p}_s(\rvx)+\hat{q}_t(\rvx))^2}d\rvx
\end{equation}
is a re-weighted bounds on the total variations between $\hat{p}_s(\rvx)$ and $\hat{q}_t(\rvx)$
\end{enumerate}
\end{proposition}

\begin{proof}
For maximization w.r.t target adaptation domain discriminator $f'$, we have
\begin{equation}
\begin{split}
&\E_{\hat{S}}\log(\sigmoid_{h_f}\circ f')+\E_{\hat{T}}\log(1-\sigmoid_{h_f}\circ f')\\
=&\int_{\rvx} \hat{p}_s(\rvx)\log(\sigmoid_{h_f}\circ f') + \hat{q}_t(\rvx)\log(1-\sigmoid_{h_f}\circ f')d\rvx 
\end{split}
\end{equation}
As we relaxed the restriction on $\sigmoid_{h_f}\circ f'$, we could find that the maximization of $p(x)\log(\sigmoid_{h_f}\circ f_t) + q(x)\log(1-\sigmoid_{h_f}\circ f')$ could be satisfied on every 
$x\in \sD$ as $\sigmoid_{h_f}\circ f'$ reaches
 \begin{equation}
 \label{optimaldisc}
 \sigmoid_{h_f}\circ f'(\rvx) = \frac{\hat{p}_s(\rvx)}{\hat{p}_s(\rvx)+\hat{q}_t(\rvx)}
 \end{equation}
 The above optimal value of $\sigmoid_{h_f}\circ f'(\rvx)$ could be derived from simple calculus.
 
 Then we analyze the maximization bounds w.r.t $\hat{S}$ and $\hat{T}$ on the equilibrium condition of target adaptation domain discriminator $f'$. By substituting the equilibrium condition of (\ref{optimaldisc}) into (\ref{mincond})
 \begin{equation}
 \begin{split}
 D = &\E_{\hat{S}}\log(\frac{1}{2}\sigmoid_{h_f}\circ f' + \frac{1}{2}\sigmoid_{h_f}\circ f_0)+\E_{\hat{T}}\log(1-\frac{1}{2}\sigmoid_{h_f}\circ f_0 - \frac{1}{2}\sigmoid_{h_f}\circ f')\\
 =&\E_{\hat{S}}\log(\frac{\hat{S}}{2(\hat{S}+\hat{T})} + \frac{1}{2}\sigmoid_{h_f}\circ f_0)+\E_{\hat{T}}\log(\frac{1}{2}-\frac{1}{2}\sigmoid_{h_f}\circ f_0 + \frac{\hat{T}}{2(\hat{S}+\hat{T})})
 \end{split}
 \end{equation}
 Using first order Taylor expansion, we have
 \begin{equation}
 \begin{split}
  D &= \underbrace{\E_{\hat{S}}\log(\frac{1}{4} + \frac{\hat{S}}{2(\hat{S}+\hat{T})}) +\E_{\hat{T}}\log(\frac{1}{4} + \frac{\hat{T}}{2(\hat{S}+\hat{T})})}_{L_1} \\
  &\underbrace{- \E_{\hat{S}}\frac{4(\hat{S}+\hat{T})}{3\hat{S}+\hat{T}}(\frac{1}{4}-\frac{1}{2}\sigmoid_{h_f}\circ f_0) + \E_{\hat{T}}\frac{4(\hat{S}+\hat{T})}{3\hat{T}+\hat{S}}(\frac{1}{4} - \frac{1}{2}\sigmoid_{h_f}\circ f_0)}_{L_2}
 \end{split}
 \end{equation}
 As we depose the $D$ into term $L_1$ and $L_2$, we could further write $L_1$ as
 \begin{equation}
 \begin{split}
     L_1 & = \E_{\hat{S}}\log(\frac{3\hat{S}+\hat{T}}{4(\hat{S}+\hat{T})}) +\E_{\hat{T}}\log(\frac{3\hat{T}+\hat{S}}{4(\hat{S}+\hat{T})}) \\
     & = -4\E_{\frac{1}{2}\hat{S}+\frac{1}{2}\hat{T}}\log(\frac{3\hat{S}+\hat{T}}{4(\hat{S}+\hat{T})})+ 4\E_{\frac{3}{4}\hat{S}+\frac{1}{4}\hat{T}}\log(\frac{3\hat{S}+\hat{T}}{4(\hat{S}+\hat{T})})\\
     & - 4\E_{\frac{1}{2}\hat{T}+\frac{1}{2}\hat{S}}\log(\frac{3\hat{T}+\hat{S}}{4(\hat{S}+\hat{T})}))+ 4\E_{\frac{3}{4}\hat{T}+\frac{1}{4}\hat{S}}\log(\frac{3\hat{T}+\hat{S}}{4(\hat{S}+\hat{T})})\\
     & = -4\E_{\frac{1}{2}\hat{S}+\frac{1}{2}\hat{T}}\log(\frac{1}{8}\frac{\frac{3}{4}\hat{S}+\frac{1}{4}\hat{T}}{\frac{1}{2}\hat{S}+\frac{1}{2}\hat{T}}) + 4\E_{\frac{3}{4}\hat{S}+\frac{1}{4}\hat{T}}\log(\frac{1}{8}\frac{\frac{3}{4}\hat{S}+\frac{1}{4}\hat{T}}{\frac{1}{2}\hat{S}+\frac{1}{2}\hat{T}})\\
     &-4\E_{\frac{1}{2}\hat{S}+\frac{1}{2}\hat{T}}\log(\frac{1}{8}\frac{\frac{3}{4}\hat{T}+\frac{1}{4}\hat{S}}{\frac{1}{2}\hat{S}+\frac{1}{2}\hat{T}}) + 4\E_{\frac{3}{4}\hat{T}+\frac{1}{4}\hat{S}}\log(\frac{1}{8}\frac{\frac{3}{4}\hat{T}+\frac{1}{4}\hat{S}}{\frac{1}{2}\hat{S}+\frac{1}{2}\hat{T}})\\
     &=4KL(\frac{3}{4}\hat{T}+\frac{1}{4}\hat{S}||\frac{1}{2}\hat{T}+\frac{1}{2}\hat{S}) + 4KL(\frac{1}{2}\hat{T}+\frac{1}{2}\hat{S}||\frac{3}{4}\hat{T}+\frac{1}{4}\hat{S})\\
     &+ 4KL(\frac{3}{4}\hat{S}+\frac{1}{4}\hat{T}||\frac{1}{2}\hat{T}+\frac{1}{2}\hat{S}) + 4KL(\frac{1}{2}\hat{T}+\frac{1}{2}\hat{S}||\frac{3}{4}\hat{S}+\frac{1}{4}\hat{T}) 
 \end{split}
 \end{equation}
Next, the term $L_2$ could be treated as
\begin{equation}
\begin{split}
L_2 &= \int_{\rvx}(1-2\sigmoid_{h_f}\circ f_0(\rvx))\frac{\hat{p}_s(\rvx)+\hat{q}_t(\rvx)}{3\hat{p}_s(\rvx)+\hat{q}_t(\rvx)})\hat{p}_s(\rvx)d\rvx -\int_{\rvx}(1-2\sigmoid_{h_f}\circ f_0(\rvx))\frac{\hat{p}_s(\rvx)+\hat{q}_t(\rvx)}{3\hat{q}_t(\rvx)+\hat{p}_s(\rvx)})\hat{q}_t(\rvx)d\rvx\\
&=\int_{\rvx}(1-2\sigmoid_{h_f}\circ f_0(\rvx))(\hat{p}_s(\rvx)+\hat{q}_t(\rvx))(-\frac{\hat{p}_s(\rvx)}{3\hat{p}_s(\rvx)+\hat{q}_t(\rvx)}+ \frac{\hat{q}_t(\rvx)}{3\hat{q}_t(\rvx)+\hat{p}_s(\rvx)})d\rvx\\
& = \int_{\rvx}(1-2\sigmoid_{h_f}\circ f_0(\rvx))(\hat{p}_s(\rvx)+\hat{q}_t(\rvx))\frac{\hat{q}_t(\rvx)^2 - \hat{p}_s(\rvx)^2 }{(3\hat{p}_s(\rvx)+\hat{q}_t(\rvx))(3\hat{q}_t(\rvx)+\hat{p}_s(\rvx))}d\rvx\\
& = \int_{\rvx}(1-2\sigmoid_{h_f}\circ f_0(\rvx))(\hat{q}_t(\rvx)-\hat{p}_s(\rvx))\frac{(\hat{p}_s(\rvx)+\hat{q}_t(\rvx))^2}{(3\hat{p}_s(\rvx)+\hat{q}_t(\rvx))(3\hat{q}_t(\rvx)+\hat{p}_s(\rvx))}d\rvx\\
& = \int_{\rvx}(1-2\sigmoid_{h_f}\circ f_0(\rvx))(\hat{q}_t(\rvx)-\hat{p}_s(\rvx))\frac{(\hat{p}_s(\rvx)+\hat{q}_t(\rvx))^2}{3\hat{p}_s(\rvx)^2 + 3\hat{q}_t(\rvx)^2+ 10\hat{p}_s(\rvx)\hat{q}_t(\rvx)}d\rvx \\
&=\int_{\rvx}(1-2\sigmoid_{h_f}\circ f_0(\rvx))(\hat{q}_t(\rvx)-\hat{p}_s(\rvx))\frac{(\hat{p}_s(\rvx)+\hat{q}_t(\rvx))^2}{4(\hat{p}_s(\rvx)+\hat{q}_t(\rvx))^2-(\hat{p}_s(\rvx)-\hat{q}_t(\rvx))^2}d\rvx \\
&=\int_{\rvx}(1-2\sigmoid_{h_f}\circ f_0(\rvx))(\hat{q}_t(\rvx)-\hat{p}_s(\rvx))\frac{1}{4-(\hat{p}_s(\rvx)-\hat{q}_t(\rvx))^2/(\hat{p}_s(\rvx)+\hat{q}_t(\rvx))^2}d\rvx
\end{split}
\end{equation}
\end{proof}
\section*{Checklist}

The checklist follows the references. For each question, choose your answer from the three possible options: Yes, No, Not Applicable.  You are encouraged to include a justification to your answer, either by referencing the appropriate section of your paper or providing a brief inline description (1-2 sentences). 
Please do not modify the questions.  Note that the Checklist section does not count towards the page limit. Not including the checklist in the first submission won't result in desk rejection, although in such case we will ask you to upload it during the author response period and include it in camera ready (if accepted).

\textbf{In your paper, please delete this instructions block and only keep the Checklist section heading above along with the questions/answers below.}

 \begin{enumerate}

 \item For all models and algorithms presented, check if you include:
 \begin{enumerate}
   \item A clear description of the mathematical setting, assumptions, algorithm, and/or model. [Yes]
   \item An analysis of the properties and complexity (time, space, sample size) of any algorithm. [Yes]
   \item (Optional) Anonymized source code, with specification of all dependencies, including external libraries. [Yes]
 \end{enumerate}

 \item For any theoretical claim, check if you include:
 \begin{enumerate}
   \item Statements of the full set of assumptions of all theoretical results. [Yes]
   \item Complete proofs of all theoretical results. [Yes]
   \item Clear explanations of any assumptions. [Yes]     
 \end{enumerate}

 \item For all figures and tables that present empirical results, check if you include:
 \begin{enumerate}
   \item The code, data, and instructions needed to reproduce the main experimental results (either in the supplemental material or as a URL). [Yes/No/Not Applicable]
   \item All the training details (e.g., data splits, hyperparameters, how they were chosen). [Yes]
         \item A clear definition of the specific measure or statistics and error bars (e.g., with respect to the random seed after running experiments multiple times). [No]
         \item A description of the computing infrastructure used. (e.g., type of GPUs, internal cluster, or cloud provider). [No]
 \end{enumerate}

 \item If you are using existing assets (e.g., code, data, models) or curating/releasing new assets, check if you include:
 \begin{enumerate}
   \item Citations of the creator If your work uses existing assets. [Yes]
   \item The license information of the assets, if applicable. [Yes]
   \item New assets either in the supplemental material or as a URL, if applicable. [Not Applicable]
   \item Information about consent from data providers/curators. [Not Applicable]
   \item Discussion of sensible content if applicable, e.g., personally identifiable information or offensive content. [Not Applicable]
 \end{enumerate}

 \item If you used crowdsourcing or conducted research with human subjects, check if you include:
 \begin{enumerate}
   \item The full text of instructions given to participants and screenshots. [Not Applicable]
   \item Descriptions of potential participant risks, with links to Institutional Review Board (IRB) approvals if applicable. [Not Applicable]
   \item The estimated hourly wage paid to participants and the total amount spent on participant compensation. [Not Applicable]
 \end{enumerate}

 \end{enumerate}


\end{document}